\documentclass{aims}
\usepackage{amsmath}
  \usepackage{paralist}
  \usepackage{graphics} 
  \usepackage{epsfig} 
\usepackage{graphicx}  \usepackage{epstopdf}
 \usepackage[colorlinks=true]{hyperref}
\hypersetup{urlcolor=blue, citecolor=red}
\usepackage{caption}

\def\currentvolume{X}
 \def\currentissue{X}
  \def\currentyear{200X}
   \def\currentmonth{XX}
    \def\ppages{X--XX}

\usepackage{lipsum}
\usepackage{amsfonts}
\usepackage{graphicx}
\usepackage{epstopdf}
\usepackage{algorithmic}
\usepackage{amsmath,amssymb}
\usepackage{mathtools}
\usepackage{color}
\usepackage{tikz-cd}
\usepackage{xcolor}
\usepackage[mathscr]{euscript}
\usepackage{comment}
\usepackage{array}
\usepackage{graphicx}
\usepackage{media9}

\newcommand{\N}{\mathbb{N}}
\newcommand{\R}{\mathbb{R}}

\newcommand{\bV}{\mathbf{V}}
\newcommand{\bW}{\mathbf{W}}
\newcommand{\bX}{\mathbf{X}}
\newcommand{\bY}{\mathbf{Y}}

\newcommand{\rank}{\mathrm{rank}}
\newcommand{\dgm}{\mathrm{Dgm}}

\newcommand{\vrp}[1]{\mathrm{VR}(#1)}
\newcommand{\vr}[2]{\mathrm{VR}(#1;#2)}

\newcommand{\gh}{\mathrm{GH}}
\newcommand{\ph}{\mathrm{PH}}
\newcommand{\norm}[1]{\left\lVert#1\right\rVert}

\definecolor{darkgreen}{RGB}{0,100,0}

\usepackage{tabularx}
\newcommand\setrow[1]{\gdef\rowmac{#1}#1\ignorespaces}
\newcommand\clearrow{\global\let\rowmac\relax}

\ifpdf
  \DeclareGraphicsExtensions{.eps,.pdf,.png,.jpg}
\else
  \DeclareGraphicsExtensions{.eps}
\fi

\newtheorem{theorem}{Theorem}[section]

\newtheorem{lemma}[theorem]{Lemma}

\theoremstyle{definition}
\newtheorem{definition}[theorem]{Definition}
\newtheorem{example}{Example}
\newtheorem{question}{Question}

\title[Capturing Dynamics of Time-Varying Data] 
      {Capturing Dynamics of Time-Varying\\ Data via Topology}

\author[Lu Xian and Henry Adams and Chad M. Topaz and Lori Ziegelmeier]{}


 \email{xianl@umich.edu}
 \email{henry.adams@colostate.edu}
 \email{cmt6@williams.edu}
 \email{lziegel1@macalester.edu}

\thanks{
L.X.\ was funded by Macalester College through a grant to L.Z.
H.A.\ was supported by NSF grant 1934725, DELTA: Descriptors of Energy Landscapes by Topological Analysis.
C.M.T.\ was supported by NSF grant DMS-1813752, Variational and Topological Approaches to Complex Systems.
L.Z.\ was supported by NSF grant CDS\&E-MSS-1854703, Exact Homological Algebra for Computational Topology (ExHACT)}

\thanks{$^*$ Corresponding author: Lori Ziegelmeier}

\begin{document}
\maketitle

\centerline{\scshape Lu Xian}
\medskip
{\footnotesize
 \centerline{School of Information}
   \centerline{University of Michigan}
   \centerline{Ann Arbor, MI 48109, USA}
} 

\medskip

\centerline{\scshape Henry Adams}
\medskip
{\footnotesize
 \centerline{Department of Mathematics}
   \centerline{Colorado State University}
   \centerline{Fort Collins, CO 80523, USA}
}

\medskip

\centerline{\scshape Chad M.\ Topaz}
\medskip
{\footnotesize
 \centerline{Department of Mathematics and Statistics}
   \centerline{Williams College}
   \centerline{Williamstown, MA 01267, USA}
}

\medskip

\centerline{\scshape Lori Ziegelmeier$^*$}
\medskip
{\footnotesize
 \centerline{Department of Mathematics, Statistics, and Computer Science}
   \centerline{Macalester College}
   \centerline{Saint Paul, MN 55105, USA}
}

\bigskip

 \centerline{(Communicated by the associate editor name)}

\begin{abstract}
One approach to understanding complex data is to study its shape through the lens of algebraic topology.
While the early development of topological data analysis focused primarily on static data, in recent years, theoretical and applied studies have turned to data that varies in time.
A time-varying collection of metric spaces as formed, for example, by a moving school of fish or flock of birds, can contain a vast amount of information.
There is often a need to simplify or summarize the dynamic behavior.
We provide an introduction to topological summaries of time-varying metric spaces including vineyards~\cite{cohen2006vines}, crocker plots~\cite{topaz2015topological}, and multiparameter rank functions~\cite{kim2020spatiotemporal}.
We then introduce a new tool to summarize time-varying metric spaces: a \emph{crocker stack}.
Crocker stacks are convenient for visualization, amenable to machine learning, and satisfy a desirable continuity property which we prove.
We demonstrate the utility of crocker stacks for a parameter identification task involving an influential model of biological aggregations \cite{Vicsek1995}.
Altogether, we aim to bring the broader applied mathematics community up-to-date on topological summaries of time-varying metric spaces.
\end{abstract}


\section{Introduction}

Drawing from subfields within mathematics, applied mathematics, statistics, and computer science, topological data analysis (TDA) is a set of approaches that help one understand complex data by studying its shape.
The application of TDA has contributed to the understanding of problems and systems in the natural sciences, social sciences, and humanities, including granular materials~\cite{giusti2016topological}, cancer biology~\cite{damiano2018topological}, development economics~\cite{banman2018mind}, political science~\cite{feng2019persistent}, urban analytics~\cite{feng2020spatial}, natural language processing~\cite{zhu2013persistent}, and much more.
Classic works that build and review the fundamental ideas of TDA include~\cite{Zomorodian,kaczynski2006computational,Ghrist2008barcodes,EdelsbrunnerHarer,Carlsson2009,otter2017roadmap}.
While TDA was originally developed with the study of static data in mind, in recent years, it has found fruitful application to time-evolving data, or as we will say, \emph{dynamically-varying} or \emph{time-varying metric spaces}.
For example, for a biological aggregation such as an insect swarm, the metric space of interest might be the positions and velocities of all organisms, which vary from frame to frame in the movie of an experimental trial~\cite{ulmer2019topological}.
For networked oscillators, the metric space of interest might be the phase of each oscillator in the ensemble, which, similarly, evolves in time~\cite{stolz2017persistent}.
These systems can produce massive amounts of data, and so there is sometimes a need to simplify or summarize the dynamic behavior of time-varying systems.
Here, TDA plays a role.

In this paper, we have three overarching goals.
First, we aim to provide a lay reader in the data science community with an overview of topological tools for studying time-varying metric spaces.
Second, we demonstrate an application of some of these tools to a parameter recovery problem arising in the study of collective behavior.
Finally, we present a new tool for time-varying metric spaces --- a crocker stack --- and explore its continuity properties.
Overall, we hope that our work will provide a mechanism to bring newcomers to the field up-to-date on approaches to time-varying metric spaces, including our own new contribution.

Topological methods for studying time-varying data are built on a technique called persistent homology.
\emph{Homology} provides a way to (partially) characterize the topology of an object. The characterization comes in the form of quantities called \emph{Betti numbers} $\beta_k$, where the $k$ indicates a boundary of dimension $k$ enclosing a void of dimension $k+1$.
Concretely, the number of connected components is $\beta_0$, the number of topological loops (circles) is $\beta_1$, the number of trapped volumes is $\beta_2$, and so on up in dimension.
For example, suppose we have a filled-in disk next to a hollow square next to a two-torus, that is, a hollow donut; see Figure~\ref{fig:bettiexample}.
The filled-in disk has Betti numbers $(\beta_0,\beta_1,\beta_2,\beta_3,\ldots)=(1,0,0,0,\ldots)$ because it is one object that is contractible and has no higher dimensional structure.
The hollow square has Betti numbers $(1,1,0,0,\ldots)$ because it is one object enclosing a flat void.
And finally, the torus has Betti numbers $(1,2,1,0,\ldots)$ because it is one object, is generated by two independent circles (one passing around the equator of the donut and one passing around the donut's hole), and encloses an empty volume.
Altogether, then, we have $(\beta_0,\beta_1,\beta_2,\beta_3,\ldots)=(3,3,1,0,\ldots)$ if we consider the homology of the union of these three shapes.

\begin{figure}
\centering
\includegraphics[width=4in]{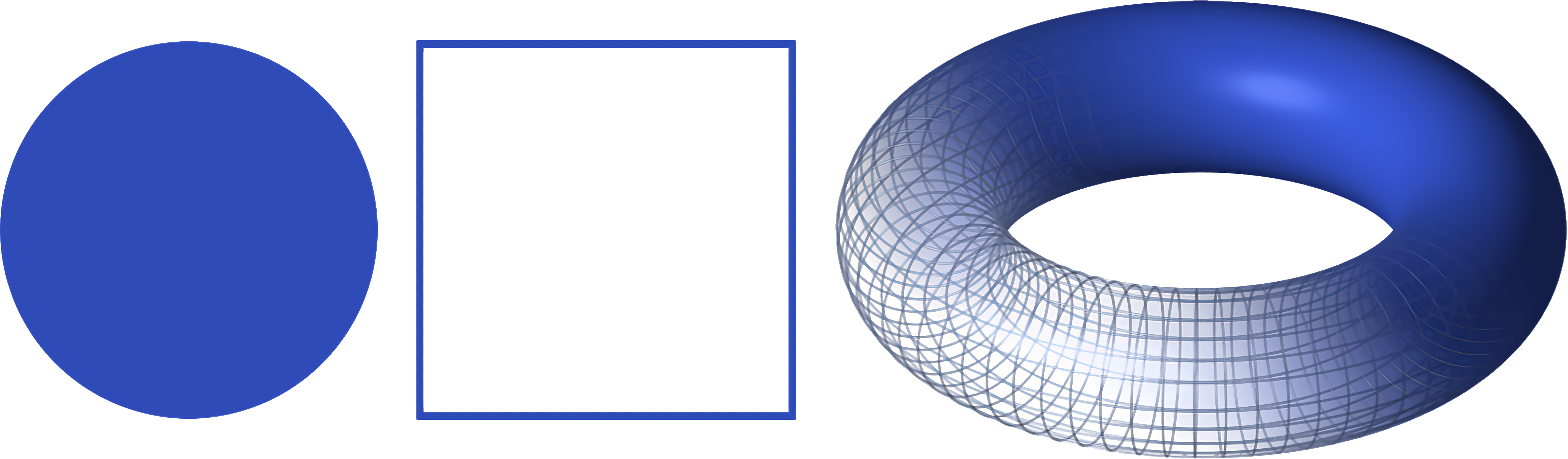}
\captionsetup{width=.9\linewidth}
\caption{
A filled-in disk (left) has Betti numbers $(\beta_0,\beta_1,\beta_2,\beta_3,\ldots)=(1,0,0,0,\ldots)$.
A hollow square (center) has Betti numbers $(1,1,0,0,\ldots)$.
A hollow two-torus (right) has Betti numbers $(1,2,1,0,\ldots)$.
If we consider the union of these three shapes, the Betti numbers are $(\beta_0,\beta_1,\beta_2,\beta_3,\ldots)=(3,3,1,0,\ldots)$. Image of the torus taken from Wikimedia Commons \mbox{\url{https://commons.wikimedia.org/wiki/File:Torus.svg}}, available for reuse under CCA BY-SA 3.0.}
\label{fig:bettiexample}
\end{figure}

In the example above, we have used idealized shapes such as a disk, hollow square, and torus.
However, we want to also think about the topological properties of a data set, rather than only idealized shapes.
Suppose we have $N$ data points in $\mathbb{R}^m$ and we want to know if this set of points has structure that we cannot see by eye.
We can build a \emph{simplicial complex} (in particular, we describe the process to construct a \emph{Vietoris-Rips} simplicial complex, but others exist as well) out of the data by placing an $m$-dimensional ball of radius $\varepsilon/2$ around each point, forming a $k$-simplex whenever $k+1$ points are pairwise within $\varepsilon$ (\emph{i.e.}\ the balls pairwise intersect), and then, calculating the Betti numbers of the object formed. 
Of course, the values of the Betti numbers will depend on our choice of $\varepsilon$.
For $\varepsilon$ approaching zero, all the balls will still be separated, and we have $\beta_0=N$ with no higher dimensional topological features.
For $\varepsilon$ approaching infinity, all the balls will overlap and we will have a giant, solid mass
with $\beta_0 = 1$ and no higher dimensional features.
At intermediate values of $\varepsilon$, the structure of the simplicial complex may well be sensitive to $\varepsilon$, and one may see topological holes of various dimensions that are born and die as $\varepsilon$ varies.
The \emph{persistent} part of persistent homology refers to calculating homology over a range of values of $\varepsilon$ and studying how topological features persist or vary.

Thus far, we have been discussing static data. Three topological summaries of time-varying data are vineyards~\cite{cohen2006vines}, crocker plots~\cite{topaz2015topological}, and multiparameter rank functions~\cite{kim2020spatiotemporal}.
As we will explain in more detail later, a \emph{vineyard} is a 3D representation of persistent homology over time.
A \emph{crocker plot} is a 2D image that displays the topological information at all scales and times simultaneously, albeit in a manner that does not necessarily elucidate persistence.
Finally, a multiparameter rank function is a computable invariant with appealing theoretical properties, providing lower bounds on strong notions of distance between dynamic metric spaces.
The three topological summaries we have mentioned here will be presented in more detail in Section~\ref{sec:related-work}.

We introduce a new topological summary called a \emph{crocker stack}, a 3D data structure that is an extension of the 2D crocker plot. The crocker stack has three important and useful features.
First, as we will prove in Section~\ref{sec:continuity}, it satisfies a continuity property.
Roughly, a small perturbation of time-varying data results in a small perturbation of the crocker stack.
Second, the crocker stack inherits appealing interpretability properties of the crocker plot.
A crocker stack is convenient for visualization purposes since topological information for all times and at all scales in the data set is displayed as a single 2D plot, and a third dimension captures the persistence of topological features; see Figure~\ref{fig:pseudo-crockerstack} for an example.
Third, a crocker stack can be discretized and converted into a feature vector, which can be used as input to machine learning tasks. We explain crocker stacks in detail in Section~\ref{sec:stacks}.

\begin{figure}[ht]
\centering
\includegraphics[width=4.6in]{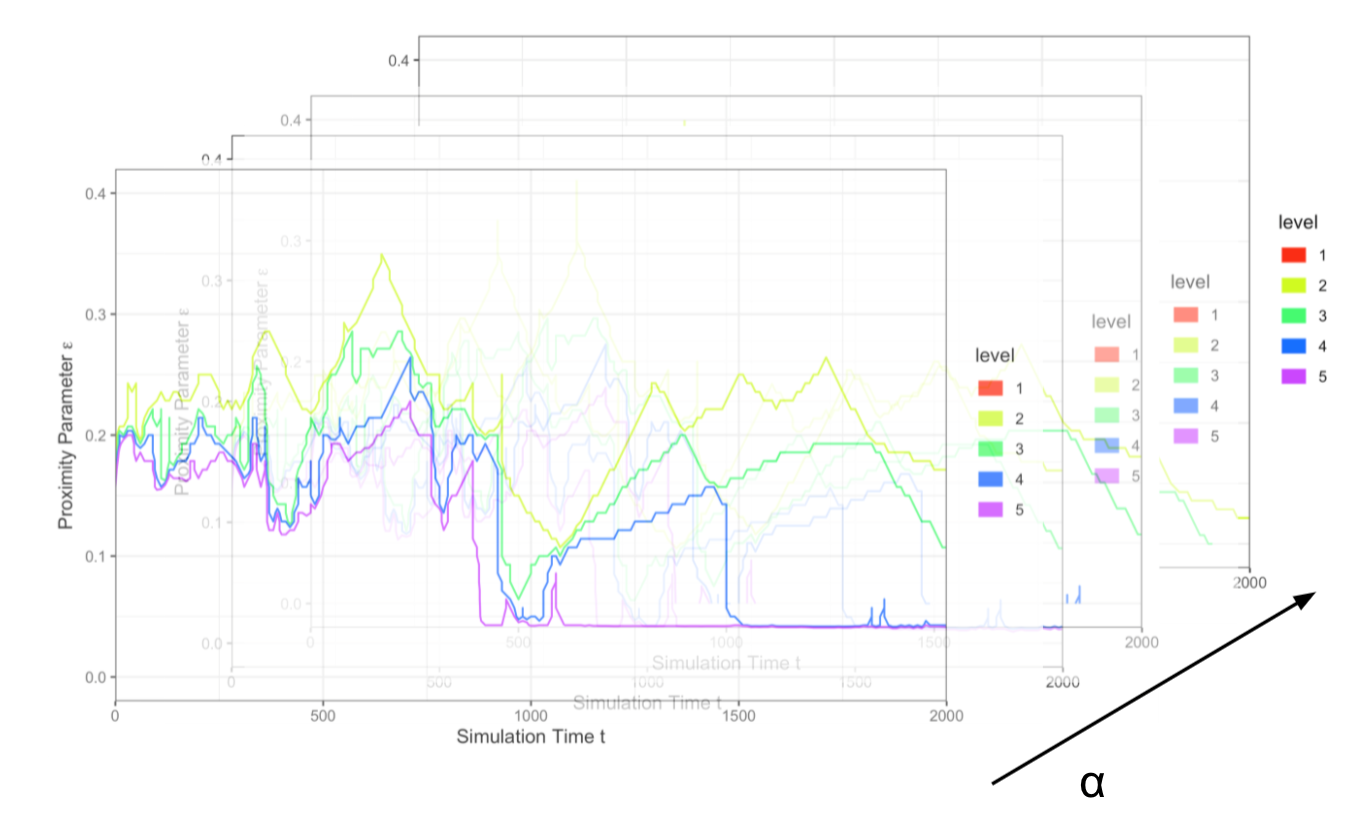}
\captionsetup{width=.9\linewidth}
\caption{An illustrative example of a crocker stack for $H_0$ computed from a simulation from the Viscek model with noise parameter $\eta=0.02$; see Section~\ref{sec:vicsek}.
The variable $\alpha$ is a smoothing parameter which captures the persistence of topological features.}
\label{fig:pseudo-crockerstack}
\end{figure}

The rest of this paper is organized as follows.
Section~\ref{sec:preliminaries} synthesizes information on persistence modules and metric spaces, thus providing the basic topological background underpinning this work in TDA.
Section~\ref{sec:related-work} reviews existing frameworks for studying time-varying metric spaces, namely vineyards, crocker plots, and multiparameter rank functions.
Section~\ref{sec:stacks} presents our new topological summary, crocker stacks.
In Section~\ref{sec:vicsek}, we use crocker plots and crocker stacks to study a seminal model of collective behavior: the Vicsek model~\cite{Vicsek1995}.
Specifically, we show that when used as inputs to machine learning algorithms for a parameter recovery task, crocker plots and stacks outperform more traditional summaries of dynamical behavior, drawn from physics.
Section~\ref{sec:distances} reviews notions of distances between metric spaces and persistence modules, which are necessary background for Section~\ref{sec:continuity}, in which we study the continuity properties of crocker stacks.
We conclude and describe possible future work in Section~\ref{sec:conclusion}.

\section{Preliminaries}\label{sec:preliminaries}

\subsection{Persistence modules}

The construction of persistent homology (as detailed in~\cite{Carlsson2009, EdelsbrunnerHarer, Ghrist2008barcodes}) begins with a \emph{filtration} $X(0) \subseteq X(1) \subseteq \ldots \subseteq X(n)$, a nested sequence of topological spaces.
These spaces are often simplicial complexes; we identify simplicial complexes with their geometric realizations.
As an example, a sequence of \emph{Vietoris--Rips complexes} parameterized by scale parameter $\varepsilon_i$ is a filtration.
Given a fixed point cloud $X$ (or, more generally, a metric space $X$) and an increasing sequence of parameter values $\varepsilon_i$ for $i\in\{0,1,\ldots, n\}$, denote a sequence of Vietoris--Rips complexes as $X(i):=\vr{X}{\varepsilon_i}$ as $i$ varies.
Here, the Vietoris--Rips complex $\vr{X}{\varepsilon_i}$ at scale $\varepsilon_i$ is the abstract simplicial complex with vertex set $X$, and with $k$-simplices corresponding to $k+1$ points in $X$ which are pairwise within distance $\varepsilon_i$.
For $\varepsilon_0\le\varepsilon_1\le\ldots\le\varepsilon_n$, we have inclusions
\[\vr{X}{\varepsilon_0} \hookrightarrow  \vr{X}{\varepsilon_1}  \hookrightarrow \ldots \hookrightarrow \vr{X}{\varepsilon_n}.\]
We apply $k$-dimensional homology $H_k$ (with coefficients in a field) to such a filtration.
This generates a vector space $H_k(X(i))$ for each space $X(i)$, whose dimension (or rank) is the $k$-th Betti number $\beta_k$.
Furthermore, the application of homology assigns, to each inclusion between topological spaces, a linear map between homology vector spaces~\cite{Zomorodian}.

The $k$-dimensional persistent homology of the filtration $X(0) \subseteq X(1) \subseteq \ldots \subseteq X(n)$ refers to the image of induced homomorphisms $H_k(X(i)) \xrightarrow{} H_k(X(j))$ for any $i\le j$.
The sequence $H_k(X(1)) \xrightarrow{} H_k(X(2)) \xrightarrow{} \ldots \xrightarrow{} H_k(X(n))$ is a \emph{persistence module} denoted by $V$, with $V(i)=H_k(X(i))$ for all $i$.
The persistence module for a fixed homological dimension $k$ is known as the \emph{$k$-dimensional persistent homology} (PH) of a filtration.
Persistent homology crucially relies on the fact that homology is a \emph{functor}, which means that an inclusion $X(i) \hookrightarrow X(j)$ indeed induces a map $H_k(X(i))\to H_k(X(j))$ on homology.
More generally, we refer to any sequence of vector spaces and linear maps $V(0)\to V(1)\to \ldots\to V(n)$ as a \emph{persistence module $V$}, whether or not the vector spaces arise from homology.

\subsection{Persistence diagrams and barcodes}

Persistence diagrams and barcodes each provide a way to display the evolution of topological features in a filtration.
In the former, a collection of points in the extended plane $\mathbb{R}^2$ is drawn.
If a homology class is born at $X(i)$ and dies at $X(j)$, we represent this homology class by a single point at the two coordinates $(i, j)$ in the \emph{persistence diagram}; see Figure~\ref{fig:barcode}~(Left).
Since we have $i \le j$ for all such points $(i, j)$, these points lie on or above the diagonal.
Points which remain throughout the entirety of the filtration are said to last to infinity.
(A technical point is that for later convenience when defining bottleneck distances in Definition~\ref{def:bottleneck}, the persistence diagram also includes each point along the diagonal, which can be interpreted as a feature that is born and dies simultaneously, with infinite multiplicity.)
We denote a persistence diagram of a persistence module $V$ as $\dgm_k(V)$ for each homology dimension $k$.

\begin{figure}[ht]
\centering
\includegraphics[width=4.6in]{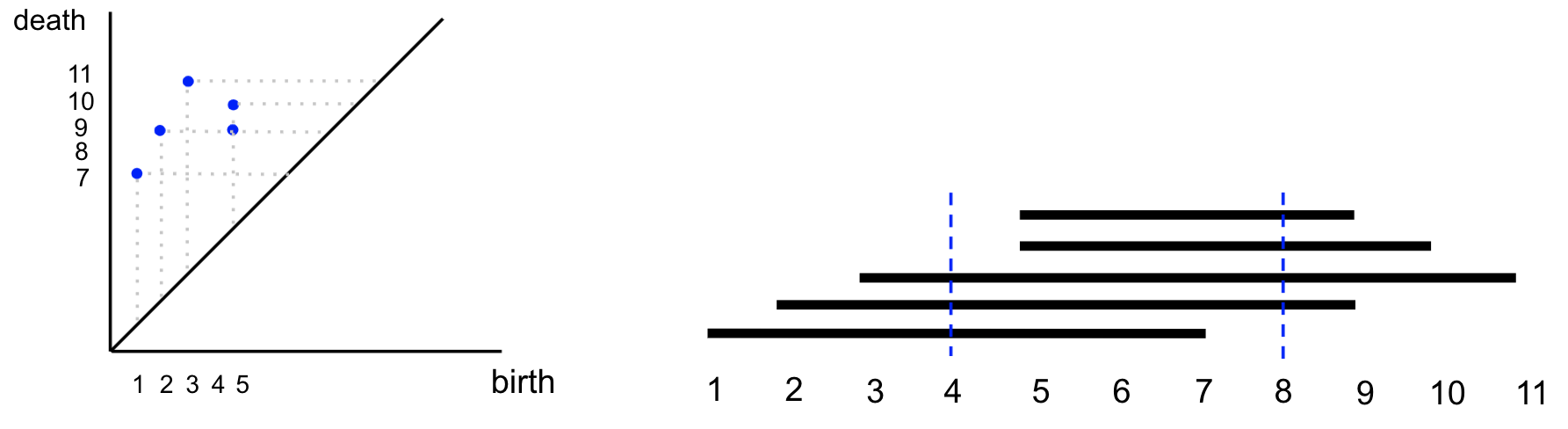}
\captionsetup{width=.9\linewidth}
\caption{
Persistence diagram (Left) and corresponding persistence barcode (Right).
Let the persistence barcode $V$ consist of the intervals $[1,7]$, $[2,9]$, $[3,11]$, $[5,10]$, $[5,9]$, and let $4 = i \le j = 8$.
Then $\rank(V(4) \to V(8))$ is two since there are two intervals in the persistence diagram that contain the interval $[i,j]=[4,8]$.
}
\label{fig:barcode}
\end{figure}

Similarly, in the \emph{barcode} representation, there is a distinct \emph{interval} (\emph{i.e.}\ a bar) corresponding to a homology class persisting over a range of scales; see Figure~\ref{fig:barcode}~(Left).
The interval begins at scale $i$, when the feature is born, and ends at scale $j$, when the feature dies.
The Betti number $\beta_k$ at scale $i$ is the number of distinct bars that intersect the vertical line through $i$.

\subsection{The rank invariant}\label{ssec:rankInvariant}

For $i \le j$, the \emph{rank} of the map $V(i) \to V(j)$, denoted $\rank(V(i) \to V(j))$, is the number of intervals in the persistence barcode that contain the interval $[i,j]$.
In other words, $\rank(V(i) \to V(j))$ is the number of features that are born before scale $i$ and die after scale $j$.
For example, suppose the persistence barcode $V$ consists of the intervals $[1,7]$, $[2,9]$, $[3,11]$, $[5,10]$, $[5,9]$, and let $4 = i \le j = 8$ (see Figure~\ref{fig:barcode}).
Then $\rank(V(4) \to V(8))$ is two since there are two intervals ($[2,9]$ and $[3,11]$) in the persistence diagram that contain the interval $[i,j]=[4,8]$.
The function $(i,j) \mapsto \mathrm{rank}(V(i) \to V(j)) \in \mathbb{N}$, for all choices of $i\le j$, is called the \emph{rank invariant}.

\subsection{The bottleneck distance}
\label{ssec:bottleneck}
Let $\dgm_k(V)$ and $\dgm_k(W)$ be two persistence diagrams associated with two persistence modules $V$ and $W$.
The distance between two points $x = (x_1, x_2)$ in $\dgm_k(V)$ and $y = (y_1, y_2)$ in $\dgm_k(W)$ is given by the $L_\infty$ distance  $\norm{x-y}_{\infty} = \max\{| x_1 - y_1|, |x_2 - y_2| \}$.

\begin{definition}
\label{def:bottleneck}
The \emph{bottleneck distance} between the two persistence diagrams $V$ and $W$ is computed by taking the supremum of the $L_{\infty}$ distance between matched points and then taking the infimum over all bijections $h\colon \dgm_k(V) \rightarrow \dgm_k(W)$~\cite{EdelsbrunnerHarer}:
\[d_b(\dgm_k(V), \dgm_k(W)) = \inf_{h} \sup_{v \in \dgm_k(V)} \norm{v - h(v)}_\infty.\]
\end{definition}
Note that such a bijection $h$ always exists because we have defined a persistence diagram to contain each point on the diagonal with infinite multiplicity.

Later in the paper, we use the notation $d_b(\ph(\vrp{X}),\ph(\vrp{Y}))$ to explicitly make clear that the persistence diagrams we refer to are the Vietoris--Rips complexes corresponding to point clouds $X$ and $Y$.
The persistent homology depends on the chosen homological dimension $k$, but we make this dependency implicit and suppress $k$ from the notation as the statements we consider are often identical for any integer $k\ge 0$.

See Figure~\ref{fig:bottleneck} for an illustration of the bottleneck distance where the red points correspond to one persistence module and the blue correspond to another.
The bottleneck distance considers all matchings between the red and blue points---where unmatched points can be paired with points on the diagonal---and finds the matching that minimizes the largest distance between any two matched points.

\begin{figure}[ht]
\centering
\includegraphics[width=1.6in]{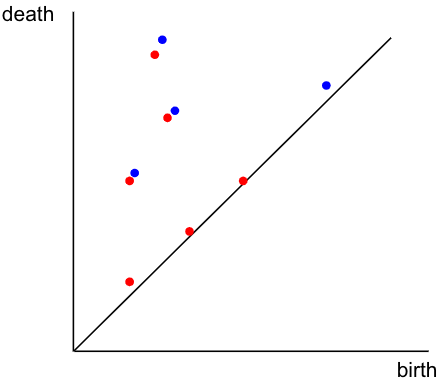}
\captionsetup{width=.9\linewidth}
\caption{Red points represent points in persistence diagram A.
Blue points represent points in persistence diagram B.
The bottleneck distance between persistence diagrams A and B is computed by first taking the $L_\infty$ distance between matched red and blue points of a bijection between A and B, and then taking the infimum over all bijections.}
\label{fig:bottleneck}
\end{figure}

\section{Related work}\label{sec:related-work}

We now survey related work on vineyards, crocker plots, and other topological and machine learning techniques for studying time-varying metric spaces.

\subsection{Vineyards}\label{ssec:vineyards}

One way to construct a topological summary of a time-varying collection of metric spaces is a \emph{vineyard}.
This summary represents a metric space that is varying over time $t\in [0,T]$ as a stacked set of persistence diagrams as time varies, with time-varying curves drawn through the persistence diagram points~\cite{cohen2006vines}.
A stacked set of persistence diagrams can be thought of as a video for visualization purposes, where each frame corresponds to a 2D persistence diagram, and the persistence diagram points evolve over the time interval $t\in [0,T]$.
See for example Figure~\ref{fig:vineyard}.

\begin{figure}[ht]
\centering
\includegraphics[width=2in]{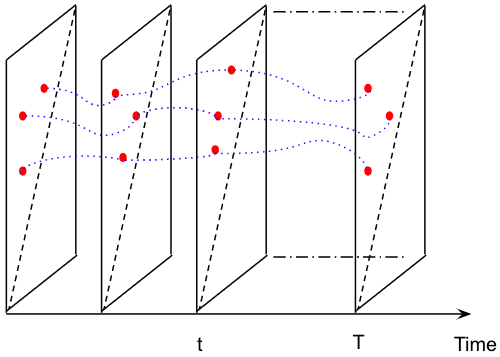}
\captionsetup{width=.9\linewidth}
\caption{Vines and vineyard.
Each red point represents a point on a persistence diagram.
Each blue curve is a vine traced out by a persistent point on time-varying persistence diagrams.
The horizontal direction denotes time.}
\label{fig:vineyard}
\end{figure}

A vineyard contains a set of curves in 3D, and each curve corresponds to a point in a time-varying persistence diagram in the 3D space $ [0,T] \times \R^2$.
That is to say, the three coordinates are the time coordinate $t$, the birth scale of a topological feature, and the death scale of a topological feature.
One of the nice properties of persistent homology is that it is \emph{stable}, which means that small perturbations in a time-varying metric space lead to small changes in the persistence diagram plots at each point in time.
Thus, when vineyards are considered only as a stacked set of persistence diagrams, they are stable.
However, the individual vines are not stable: if two vines move near to each other but then pull apart without touching, so that their persistence diagram points return to their original locations, then after a small perturbation these two vines may instead cross and have their corresponding persistence diagram points switch locations~\cite{Morozov}.
As such, throughout the paper, we will refer to a vineyard as a ``stacked set of persistence diagrams" or a ``vineyard," when the distinction between the two notions matters.
More formally, for us a stacked set of persistence diagrams is a continuous map from the interval $[0,T]$ to the space of persistence diagrams equipped with the bottleneck distance.
A vineyard furthermore contains the information of time-varying curves of the persistence diagram points.

The paper~\cite{cohen2006vines} describes the theory of computing vineyards, which are implemented in, for example, the Dionysus software package~\cite{dionysus}.
More specifically, we can compute the initial persistence diagram at time $t=0$, which has a sub-cubic running time in the number of simplices~\cite{milosavljevic2011zigzag}, and then update the vineyard as $t$ increases in a linear way (linear in the number of simplices whose orderings in the filtration get transposed).

\subsection{Crocker plots}

Given the same input as a vineyard, \emph{i.e.}\ a metric space that is varying over time $t \in [0,T]$, the crocker plot gives a topological summary that is an integer-valued function on $\R^2$, where the first input is time $t$, the second input is the scale $\varepsilon$, and the value of the function is the $k$-dimensional Betti number of the corresponding Vietoris--Rips complex~\cite{topaz2015topological}.
This function of two variables ($t$ and $\varepsilon$) can be discretized as a matrix and viewed as a contour diagram; see Figure~\ref{fig:crockerPlot} in Section~\ref{ssec:CrockerPlot} as an example.
The crocker plot is sometimes better suited for applications than a vineyard, since most scientific images are in 2D as opposed to 3D.
It can also more easily be used as a feature vector for machine learning tasks.
One drawback of a crocker plot is that it is not stable --- perturbing the dynamic metric space only slightly could produce changes  of unbounded size in the crocker plot (see Example~\ref{ex:crocker-plot-not-naively-stable} in Section~\ref{ssec:discont}).

One can think of the crocker plot as a collapsed or projected version of the vineyard: it is collapsed in the sense that it is lower dimensional (2D instead of 3D), and also in the sense that from a vineyard you can produce the corresponding crocker plot but not vice-versa.

\subsection{Time-varying metric spaces}

We now review works which analyze the topological structure of time-varying metric spaces.

In the proof-of-concept paper~\cite{topaz2015topological}, Topaz \emph{et al.}\ develop the crocker plot and apply it to four realizations of numerical simulations arising from the influential biological aggregation models of Vicsek~\cite{Vicsek1995} and D'Orsogna~\cite{D'Orsogna2006}.
Traditionally, order parameters derived from physics, such as polarization of group motion, angular momentum, etc., are calculated to assess structural differences in simulations.
The authors compare these order parameters to the topological crocker plot approach and discover that the latter reveals dynamic changes of these time-varying systems not captured by the former.

Ulmer \emph{et al.}~\cite{ulmer2019topological} use crocker plots to analyze the fit of two mathematical, random walk models developed by~\cite{NilPaiWar2013} to experimental data of pea aphid movement.
One model incorporates social interaction of the aphids, which is thought to be of importance for pea aphid movement, while the other is a control model.
The authors compare time-varying data from the models to time-varying data from the experiments using statistical tests on three metrics (order parameters commonly used in collective motion studies, order parameters that use \emph{a priori} input knowledge of the models, and the topological crocker plots).
The topological approach performs as well as the order parameters that require prior knowledge of the models and better than the ones that do not require prior knowledge, indicating that the topological approach may be useful to adopt when one has little information about underlying model mechanics.

Bhaskar \emph{et al.}~\cite{bhaskar2019analyzing} use crocker plots coupled with machine learning for parameter recovery in the model of D'Orsogna~\cite{D'Orsogna2006}.
The authors generate a large corpus of simulations with varying parameters, which result in different phenotypic patterns of the simulation.
For instance, over time, the particles may exhibit the structure of a single or double mill, collectively swarm together, or escape.
Each simulation is then transformed into a feature vector that summarizes the dynamics: either the aforementioned time series of order parameters (polarization, angular momentum, absolute angular momentum, average distance to nearest neighbors, or the concatenation of all four) or vectorized crocker plots corresponding to 0-dimensional homology $H_0$, 1-dimensional homology $H_1$, or the concatenation of the two.
The feature vectors are then fed into both supervised and unsupervised machine learning algorithms in order to deduce phenotypic patterns or underlying parameters.
In all cases, the topological approach gives far more accurate results than the order parameters, even with no underlying knowledge of the dynamics.

While these papers highlight that the crocker representation can be effective at modeling time-varying data and is amenable as a feature vector for machine learning tasks, crocker plots do not consider the persistence of topological features.
At each time step, a crocker plot summarizes only Betti numbers at each scale independently.
For example, while a $H_1$ crocker plot contains the topological information about the number of topological 1-dimensional holes at each scale, it does not encode the information of the scales when holes first appear and subsequently disappear.
As a result, crocker plots cannot show if the 1-dimensional holes at two different scales are in fact the same features or different.
Also as mentioned above, crocker plots are not stable.
In other words, small perturbations in dynamic metric spaces can produce changes of unbounded size in crocker plots, exemplified by Example~\ref{ex:crocker-plot-not-naively-stable}.
We address this problem and develop a persistent version of crocker plots, a \emph{crocker stack}, introduced in Section~\ref{sec:stacks}, which can capture persistent structural features of time-varying systems.

In contrast to the crocker plot approach, Corcoran \emph{et al.}~\cite{corcoran2017modelling} model swarm behavior of fish by computing persistent topological features using \emph{zigzag persistent homology}~\cite{ZigzagPersistence,carlsson2009zigzag}.
Briefly, a zigzag persistence module connects topological spaces via inclusion maps using either forward or backward arrows and tracks topological features through these inclusions.
The authors use inclusions from two consecutive time steps to the union of these time steps to track the evolution of features through time.
The resulting persistence diagram is transformed to a \emph{persistence landscape}~\cite{Bubenik2015}, which is a stable functional representation related to the rank invariant.
The persistence landscape exists in a normed vector space, and as such, the authors of~\cite{corcoran2017modelling} use statistical tests to cluster swarm behavior of fish into frequently occurring behaviors named flock, torus, and disordered.
However, while this method is persistent in time, to compute zigzag persistent homology for a fixed scale parameter requires \emph{a priori} knowledge of the underlying data to choose an appropriate scale.

In a more theoretical exploration, Kim and M{\'e}moli~\cite{kim2017stable} develop persistent homology summaries of time-varying data by encoding a finite dynamic metric space as a zigzag persistence module.
They prove that the resulting persistence diagram is stable under perturbations of the input dynamic graph in relation to defined distances on the dynamic graphs.
In a related research direction, these authors also consider an invariant for dynamic metric spaces in~\cite{kim2020spatiotemporal}.
They introduce a spatiotemporal filtration which can measure subtle differences between pairs of dynamic metric spaces.
By producing a 3D persistence module where one of the dimensions is not the real line but a poset, the invariant obtains higher differentiability power than a vineyard representation.
Intuitively,~\cite{kim2020spatiotemporal} considers smoothings in both time and space.
By contrast, crocker stacks focus on smoothings in space alone, with the goal of obtaining vectorizable summaries as input for machine learning tasks.

To a lesser degree, our techniques are also related to multiparameter persistence~\cite{carlsson2009computing,carlsson2009theory,cerri2013betti,chacholski2017combinatorial,lesnick2015theory,miller2017data,scolamiero2017multidimensional} as we compare how time-varying metric spaces evolve in both the parameters of time and scale.
However, unlike the scale parameter, there is not inclusion from one time step to the next, and as such, there is not an increasing filtration in time.
See also~\cite{mccleary2020edit} for a functorial model of time, discretized using cellular cosheaves.
The papers~\cite{EvasionPaths,adams2021efficient,carlsson2019parametrized,Coordinate-free,de2007coverage} consider time-varying notions of topology applied to coverage problems in mobile sensor networks.

\section{Crocker stacks}\label{sec:stacks}

We now describe the crocker plot summary of a time-varying metric space, and its extension to a crocker stack.
We give a precise definition of time-varying metric spaces and time-varying persistence modules suitable for our context, describe crocker plots and $\alpha$-smoothed crocker plots, and finally introduce crocker stacks.

\subsection{Time-varying metric spaces and persistence modules}

We use bold letters to denote time-varying objects, and non-bold letters to represent objects at a single point in time.

A \emph{time-varying metic space} $\bX=\{X_t\}_{t \in [0,T]}$ is a map $t\mapsto X_t$ from the interval $[0,T] \subseteq \R$ to the collection of all compact metric spaces.
Here $X_t$ is a single metric space at the fixed point in time $t$.
We say $\bX$ is continuous if this map $t\mapsto X_t$ is continuous with respect to the Gromov--Hausdorff distance (Section~\ref{ssec:gh}), and we say $\bX$ is finite if there is some integer $N$ such that the size of metric space $X_t$ is at most $N$ for all $t\in [0,T]$.

For example, if $Z$ is a fixed metric space (perhaps Euclidean space $Z=\R^n$), and if $x_1,\ldots,x_N\colon [0,T]\to Z$ are a collection of $N$ continuous maps into $Z$, then we can form a time-varying metric space $\bX=\{X_t\}_{t \in [0,T]}$ by letting $X_t=\{x_1(t),\ldots,x_N(t)\}$.
We note that $\bX$ is both continuous and finite.
Many of the time-varying metric spaces that we consider when studying agent-based collective motion models are constructed in this way.

Similarly, a \emph{time-varying persistence module} $\bV=\{V_t\}_{t \in [0,T]}$ is a map $t\mapsto V_t$ from the interval $[0,T] \subseteq \R$ to the collection of all persistence modules, equipped with the bottleneck distance.
Here $V_t$ is a single persistence module at the fixed point in time $t$.
We say that $\bV$ is continous if this map $t\mapsto V_t$ is continuous.

If $\bX$ is a time-varying metric space, then we can form a time-varying persistence module $\bV=\ph(\vrp{\bX})$ defined for each $t\in[0,T]$ by $V_t=\ph(\vrp{X_t})$.
Here we have fixed the homological dimension $k\ge 0$ and supressed it from the notation.
It follows from the stability of persistent homology (Section~\ref{ssec:stability-ph}) that if $\bX$ is continuous, then so is $\bV$.

We remark that a vineyard contains more information than a time-varying persistence module.
Indeed, a vineyard additionally contains a matching between the points in the persistence diagram $V_t$ and $V_{t+\varepsilon}$ for $\varepsilon>0$ sufficiently small.

\subsection{Crocker plots}
Crocker plots were originally defined on time-varying metric spaces $\bX$, after first applying Vietoris--Rips complexes and then homology to get the time-varying persistence module $\bV=\{V_t\}_{t\in [0,T]}$, where $V_t=H_k(\vr{X_t}{\varepsilon})$~\cite{topaz2015topological}.
Here, we take the more general approach and define a crocker plot for any time-varying persistence module $\bV$, regardless of its origins.
In fact, we note a crocker plot can be formed for any two parameter family of topological spaces, such as a bifiltration, not merely one varying in time. 

Let $\bV$ be a time-varying persistence module, with $V_t$ the persistence module at time $t$.
In the 2D \emph{crocker plot} of homological dimension $k$, the value at time $t$ and scale parameter $\varepsilon$ is the rank (or dimension) of the vector space $V_t(\varepsilon)$.
This function of two variables can be viewed as a contour plot, as shown in Figure~\ref{fig:crockerPlot}, which is suitable for applications as all times are displayed simultaneously.
For discretized values of $\varepsilon$ and $t$ (as for computational purposes), the ranks of $V_t(\varepsilon)$ can be represented as a matrix.
This matrix can then be vectorized and used as a feature vector for machine learning algorithms.

If $V_t$ is the $k$-dimensional homology of the Vietoris--Rips complex of a metric space $X_t$, then taking this time-varying metric space as input, we obtain a crocker plot, which again returns a topological summary that is an integer-valued function on $\R^2$. 
The rank at time $t \in [0,T]$ and scale $\varepsilon$ is then the number of ``$k$-dimensional holes at scale $\varepsilon$", also known as the Betti number $\beta_k$.
For a fixed $t$ and varying $\varepsilon$, this is equivalent to the notion of a \emph{Betti curve}.

\subsection{$\alpha$-smoothed crocker plots}

An extension of a crocker plot is an \emph{$\alpha$-smoothed} crocker plot.
When applied to a time-varying persistence module $\bV=\{V_t\}_{t\in [0,T]}$, the output of an $\alpha$-smoothed crocker plot for $\alpha \ge 0$ is the rank of the map $V_t(\varepsilon-\alpha) \to V_t(\varepsilon+\alpha)$ at time $t$ and scale $\varepsilon$.
A standard crocker plot can also be called a \emph{0-smoothed} crocker plot.
The effect of $\alpha$-smoothing is shown in Figure~\ref{fig:smoothing}.
Note that $\alpha$-smoothing can potentially reduce noise in a crocker plot.

\begin{figure}[ht]
\centering
\includegraphics[width=4in]{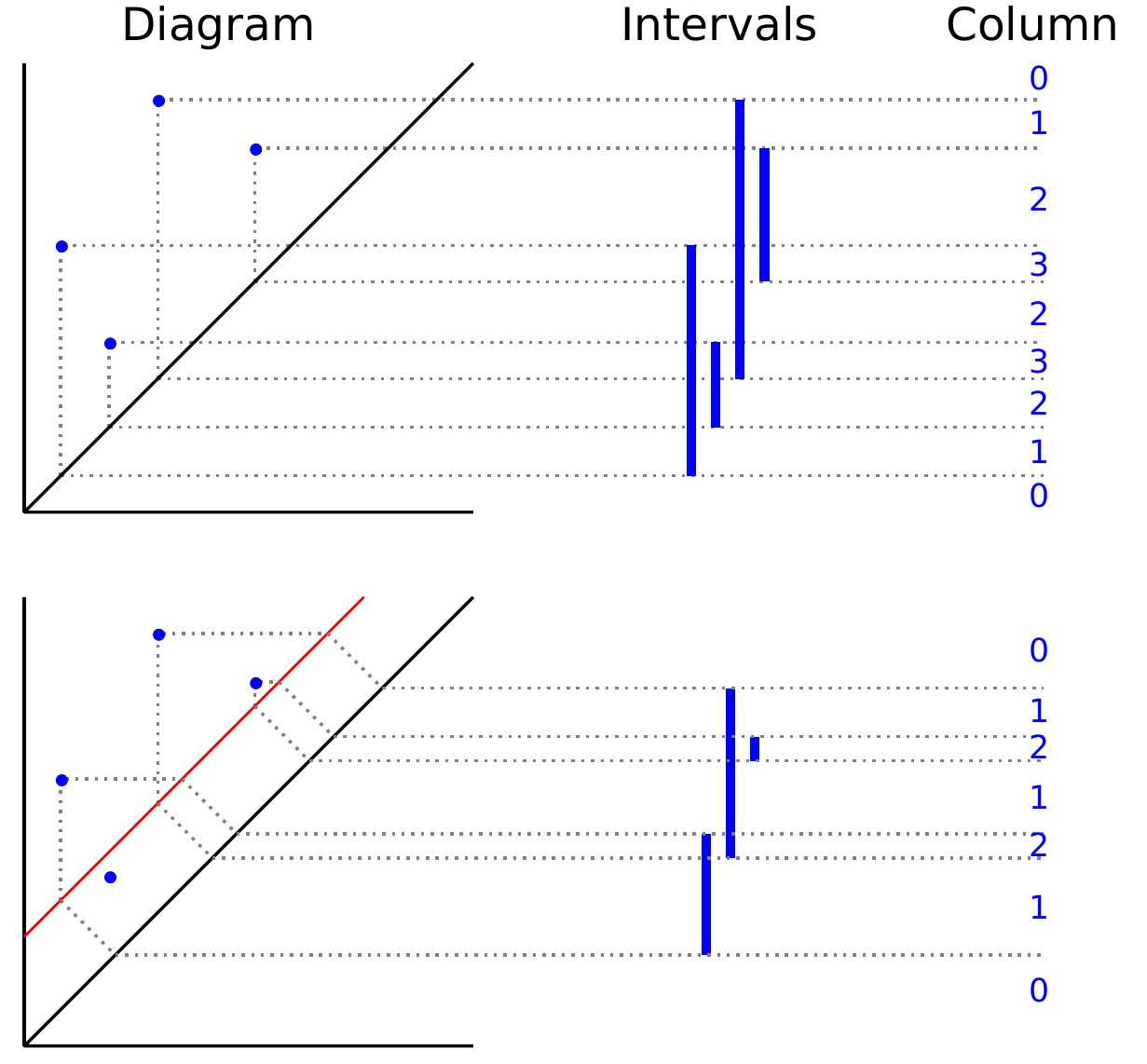}
\captionsetup{width=.9\linewidth}
\caption{The effect of $\alpha$-smoothing.
(Top) A persistence diagram, the corresponding persistence intervals (drawn vertically), and one column of a crocker plot matrix.
If we had points moving in time, then we would get a time-varying persistence diagram, a time-varying persistence barcode, and a complete crocker plot matrix (swept out from left to right as time increases).
(Bottom) A persistence diagram with the red line reflecting the choice of $\alpha$-smoothing, along with the corresponding $\alpha$-smoothed persistence intervals, and one column of an $\alpha$-smoothed crocker plot matrix.
The $y$-intercept of the diagonal red line is $2\alpha$.
All persistence diagram points under the red line are ignored under $\alpha$-smoothing.
}
\label{fig:smoothing}
\end{figure}

\subsection{The crocker stack for time-varying persistence diagrams}\label{sec:vid}

The \emph{crocker stack} is a sequence, thought of as a video, of $\alpha$-smoothed crocker plots, in which each $\alpha$-smoothed crocker plot is a frame of the video for continuously increasing $\alpha$, starting at $\alpha=0$.

\begin{definition}
\label{def:crocker-stack}
A \emph{crocker stack} summarizes the topological information of a time-varying persistence module  $\bV$ in a function $f_\bV\colon [0,T]\times[0,\infty)\times[0,\infty)\to\N$, where
\[f_\bV(t,\varepsilon,\alpha)
=\rank(V_t(\varepsilon-\alpha)\to V_t(\varepsilon+\alpha)).\]
\end{definition}

In the 3D domain $[0,T]\times[0,\infty)\times[0,\infty)$, the horizontal axis displays the time $t\in[0,T]$, and the vertical axis displays the scale $\varepsilon\in[0,\infty)$.
The persistence parameter $\alpha\in[0,\infty)$ indicates the order of the video frames.
In other words, the crocker stack is a sequence of $\alpha$-smoothed crocker plots that vary over $\alpha\ge 0$.
As examples, see Figures~\ref{fig:pseudo-crockerstack} and~\ref{fig:pseudo-crockerstackpanel}.

The crocker stack has the property that $f_\bV(t,\varepsilon,\alpha) \le f_\bV(t,\varepsilon,\alpha')$ for $\alpha \ge \alpha'$: larger $\alpha$ values require features to persist longer, which means that the crocker stack is a non-increasing function of $\alpha$.
Viewing a crocker stack as $\alpha$ increases could help one identify interesting smoothing parameter choices $\alpha$ to consider.

In Section~\ref{sec:rank}, we describe how the crocker stack and the stacked set of persistence diagrams are equivalent to one another, in the sense that either one contains the information needed to reconstruct the other.
However, crocker stacks can sometimes display the equivalent information in a more useful format.
This is because in a crocker stack, all times $t \in [0,T]$ are represented in each frame $\alpha$, whereas in a time-varying persistence diagram during a single frame $t$ one only ever sees information about that time.
We also explain in Section~\ref{sec:continuity} how crocker stacks are continuous, though, unfortunately, not in a way that would be the most immediately pertinent for machine learning applications.

\section{Experiments with the Vicsek model}
\label{sec:vicsek}

\subsection{Background}

We now turn to evaluating the utility of $\alpha$-smoothed crocker plots for applications.
This assessment centers on a parameter identification task for a mathematical model developed by Vicsek and collaborators~\cite{Vicsek1995}.
\emph{Parameter identification} refers to deducing the parameters of a model from experimental or simulation data.
The \emph{Vicsek model} is a seminal model for collective motion, in which agents attempt to align their motion with that of nearby neighbors, subject to a bit of random noise.
Concretely, the question we ask is ``given time series data output from the Vicsek model, can we recover the model parameters responsible for simulating that data?''

A key quantity in the Vicsek model is a noise parameter, $\eta$, which measures the degree of randomness in an agent's chosen direction of motion.
We generate 100 simulations for each of 15 different values of $\eta$, for a total of 1500 simulations.
We then create four different experiments, each using simulation data from a chosen subset of $\eta$ values.

An experiment consists of the following procedure.
For all of the simulations admitted to the experiment, we compute time series of feature vectors that summarize the simulation data: an order parameter from the physics literature that measures alignment of agents, $\alpha$-smoothed crocker plots, a (discretized) crocker stack, and a stacked set of persistence diagrams (in one experiment).
As a reminder, the terminology ``stacked set of persistence diagrams" is used in place of a vineyard when we do not consider the vines or curves traced out by the persistence points.
For the first three feature vectors, we compute pairwise distances between simulations using a Euclidean norm, and for the stacked set of persistence diagrams, we compute (the computationally expensive) supremum bottleneck distance.
These pairwise distances are the inputs to a machine learning algorithm, $K$-medoids clustering.
We then assess the success of clustering under different choices for the distance metric by examining how many simulations were clustered to a medoid with the same $\eta$ value.

For the remainder of this section, we will present the Vicsek model, describe the design of our numerical experiments, and present the results of using various metrics for parameter identification methods.
These results allow us to compare the efficacy of topological approaches to that of more traditional ones.

\subsection{Research design}

\subsubsection{Vicsek model} 
Cited thousands of times in the scientific literature, the Vicsek model is a seminal model for collective motion~\cite{Vicsek1995}.
This discrete time, agent-based model tracks the positions $\vec{x}_i$ and headings $\theta_i\in[0,2\pi)$ of $n$ agents in a square region with periodic boundary conditions.
The agents are seeded with uniformly random positions and uniformly random initial headings.
At each time step, an agent updates its heading and its position.
The new heading is taken to be the average heading of nearby neighbors, added to a small amount of random noise drawn from the uniform distribution $(-\eta/2,\eta/2)$.
More explicitly, the model is:
\[\theta_i(t+\Delta t) = \frac{1}{N} \left( \sum_{| \mathbf{x}_i - \mathbf{x}_j| \leq R}^{} \theta_j(t) \right) + U(-\tfrac{\eta}{2},\tfrac{\eta}{2}).\]
Above, $R$ is the distance threshold under which nearby neighbors interact, $N$ is the number of neighbors within distance $R$, and $U$ is the uniform distribution.
See Figure~\ref{fig:Vicsek} for an illustration of this model.
\begin{figure}[ht]
\centering
\includegraphics[width=2.2in]{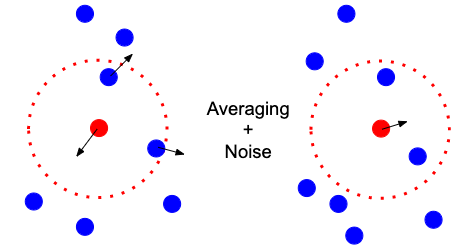}
\captionsetup{width=.9\linewidth}
\caption{To update the heading of an agent (in red) according to the Vicsek model, we first find the nearby neighbors of an agent within a radius $R$ (denoted by the red circle) and then take the average of its neighbors' headings, plus some noise.}
\label{fig:Vicsek}
\end{figure}
The Vicsek model updates $\theta_i$ and $\vec{x}_i$ according to an Euler's method type of update, that is, 
\[\mathbf{v}_i(t + \Delta t) = v_0 (\cos \theta_i (t+\Delta t), \sin \theta_i (t+\Delta t))\]
\[\mathbf{x}_i(t+\Delta t) = \mathbf{x}_i(t) + \mathbf{v}_i(t+\Delta t)\Delta t.\]
Note that all particles move with the same constant speed $v_0$.

\subsubsection{Parameters and simulations}\label{ssec:parameters}

The parameters in the Vicsek model are the number of agents $n$, the radius of alignment interaction $R$, the level of noise $\eta$, the length of a side of the periodic domain $\ell$, the agent speed $v_0$, and the time step $\Delta t$.
In the literature on this model, $\Delta t$ and $R$ are typically set to unity without loss of generality since one may rescale time and space.
This leaves the parameters $n$, $\eta$, $\ell$, and $v_0$.
Due to the periodic domain and finite sensing radius of the agents, these collapse effectively into three parameters, namely $v_0$, $\eta$, and an agent density $\rho := n/\ell^2$.

For each simulation, we fix length $\ell=25$ and $n=300$ agents, which gives that $\rho = 0.48$.
We also take speed $v_0 = 0.03$.
The remaining noise parameter $\eta$ is the value that we hope to predict from the output of the simulation by using machine learning.
We generate 100 simulations for each of 15 $\eta$ values:
\[\eta\in\{ 0.01, 0.02, 0.03, 0.05, 0.1, 0.19, 0.2, 0.21, 0.3, 0.5, 1, 1.5, 1.9, 1.99, 2\}.\]
We compare four methods of predicting $\eta$: order parameters, as described in Section~\ref{ssec:OP}, crocker plots and stacks, which were introduced in Section~\ref{ssec:CrockerPlot}, and in the case of one experiment, stacked sets of persistence diagrams, as described in Section~\ref{ssec:vineyards}.

Simulation data sets generated from the Vicsek model with different noise parameters $\eta$ include four variables: time $t$, position coordinates $x$ and $y$, and heading $\theta$.
We then transform the heading $\theta$ to velocity $\textbf{v} = (v_x, v_y)$.
This velocity can be used for computing the alignment order parameter for each simulation at each time step $t$, defined in Section~\ref{ssec:OP}.

To compute the persistent homology of simulations for crocker plots, discussed in more detail in Section~\ref{ssec:CrockerPlot}, we use both the 2-dimensional position data and the heading of each agent at each time step.
The $x$ and $y$ positions are scaled by $1/\ell=1/25$, the length of the box, and the heading is scaled by $1/2\pi$, in order for all data to be in the range from 0 to 1.
  
In Section~\ref{ss:experiments}, we will design four machine learning experiments by including different combinations of the 15 $\eta$ noise values.
We choose time steps 1, 10, and 40 for different experiments and will show the effect of the size of time steps on clustering accuracy.

Vicsek's original work~\cite{Vicsek1995} identifies several different possible qualitative behaviors of the system.
To quote him directly, ``For small densities and noise, the particles tend to form groups moving coherently in random directions\ldots
at higher densities and noise, the particles move randomly with some correlation.
For higher density and small noise, the motion becomes ordered.''

\subsubsection{Alignment order parameter} \label{ssec:OP}

Order parameters are a common way to measure the level of global synchrony in an agent-based model over time.
The alignment order parameter
\[
\varphi(t) = \frac{1}{n v_0} \left\| \sum_{i = 1}^{n} \mathbf{v}_i(t) \right\|
\]
is defined as the normalized average of the velocity vectors at time $t$, where $n$ is the number of particles in the model, as above.
This creates a time series recording the degree to which particles are aligned at each time, with 1 indicating a high degree of alignment and 0 indicating no alignment.
Calculating the order parameter is a conventional approach derived from physics.
We want to compare this approach to topological approaches on the task of parameter identification, by clustering simulations of biological aggregations with the Vicsek model into corresponding groups based on the noise parameter.

Figure~\ref{fig:orderParameter} displays the alignment order parameters for three simulations with different noise parameters $\eta=0.01$, $1$, $2$, and their changes over time.
At time zero, all three values of the order parameter are quite low due to the random initialization of headings.
All three simulations exhibit some alignment over time, with the smaller noise parameters reflecting a higher level of synchrony, consistent with Vicsek's description above.

\begin{figure}[ht]
\centering
\includegraphics[width=3.5in]{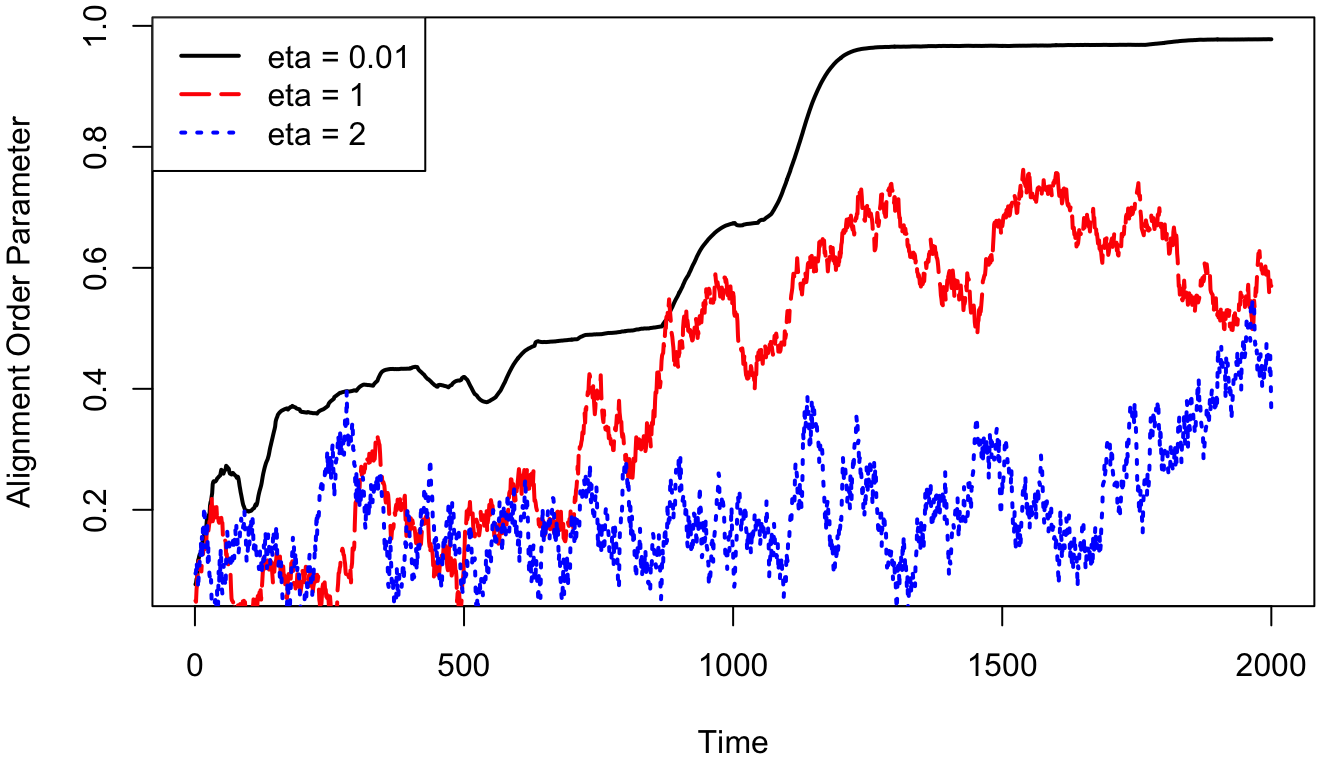}
\captionsetup{width=.9\linewidth}
\caption{A plot of order parameters for three simulations of the Vicsek model with different noise parameters $\eta$.
For smaller values of $\eta$, particles become more aligned, \emph{i.e.}\ move in the same direction, over time.}
\label{fig:orderParameter}
\end{figure}

\subsubsection{Crockers}\label{ssec:CrockerPlot}

Crocker plots, $\alpha$-smoothed crocker plots, and crocker stacks (collectively referred to as crockers) are three inputs that we will test in our machine learning experiments. 
We compute persistent homology using the R software package ``TDA''~\cite{R-TDA}, subsampling to every 10 time steps to speed up computations.
At each subsampled time, we compute the persistent homology of Vietoris--Rips filtered simplicial complexes, with vertex set given by the scaled location and heading of each agent.
We then compute the crockers from these persistent homology intervals over all times.

We make the following choices in the computations of our crockers.
\begin{itemize}
\item We only compute persistent homology in dimensions zero and one.
\item We compute the Vietoris--Rips filtration up to scale parameter $\varepsilon=0.35$.
These computations are discretized to consider 50 equally-spaced values of $\varepsilon$ between $\varepsilon=0$ and $\varepsilon=0.35$, inclusive.
\item We consider 18 different smoothing values: from $\alpha=0$ to $\alpha=0.17$ inclusive, with steps of size $0.01$ in between.
\end{itemize}

\begin{figure}[ht]
\centering
\includegraphics[width=4in]{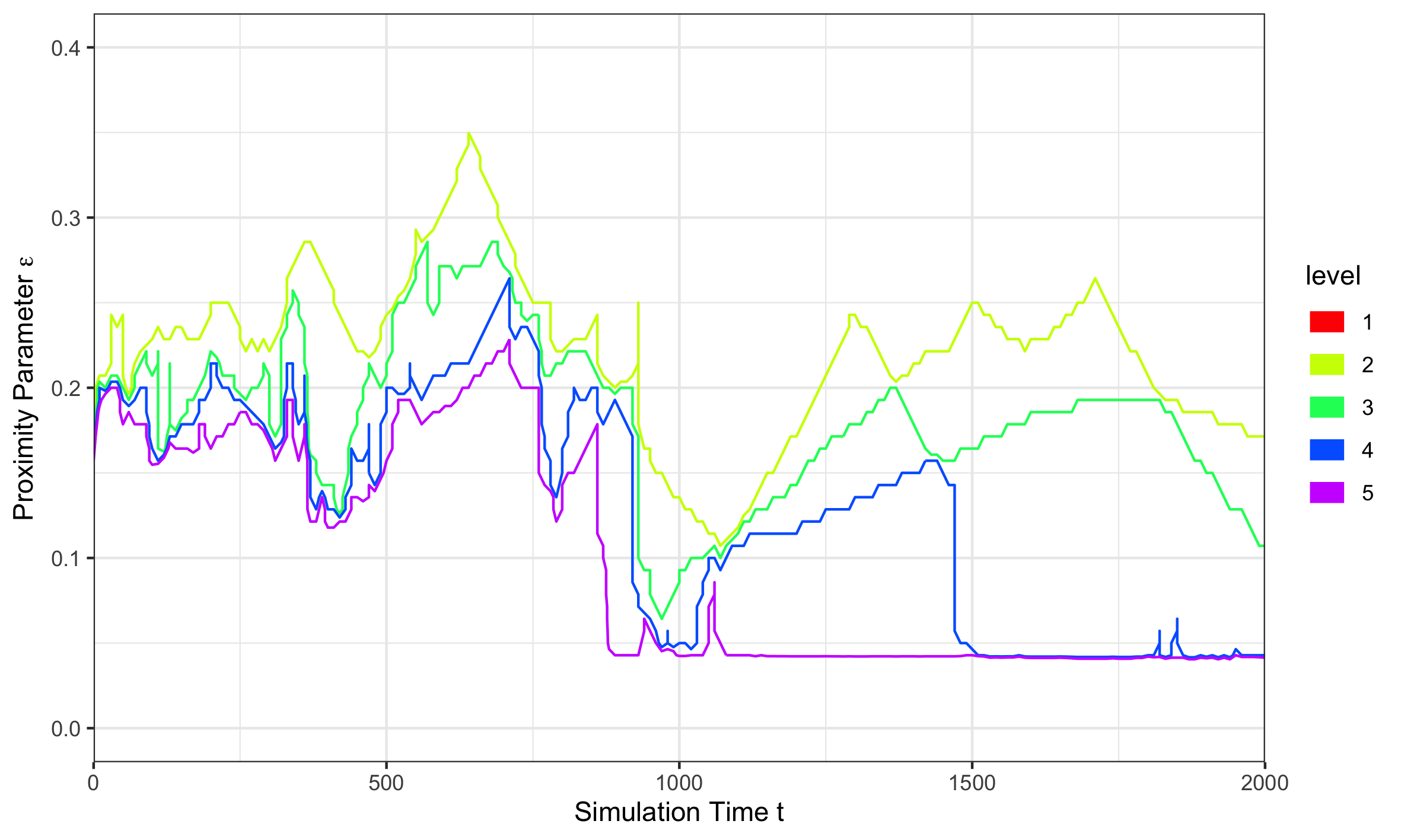}
\captionsetup{width=.9\linewidth}
\caption{An example $H_0$ crocker plot of a simulation from the Viscek model with noise parameter $\eta=0.02$.
This is the same as an $\alpha$-cross section of a crocker stack when $\alpha$ = 0.
Contour levels above $\beta_0=6$ are not displayed and are interpreted as noise.
}
\label{fig:crockerPlot}
\end{figure}

\begin{figure}[ht]
\centering
\includegraphics[width=5in]{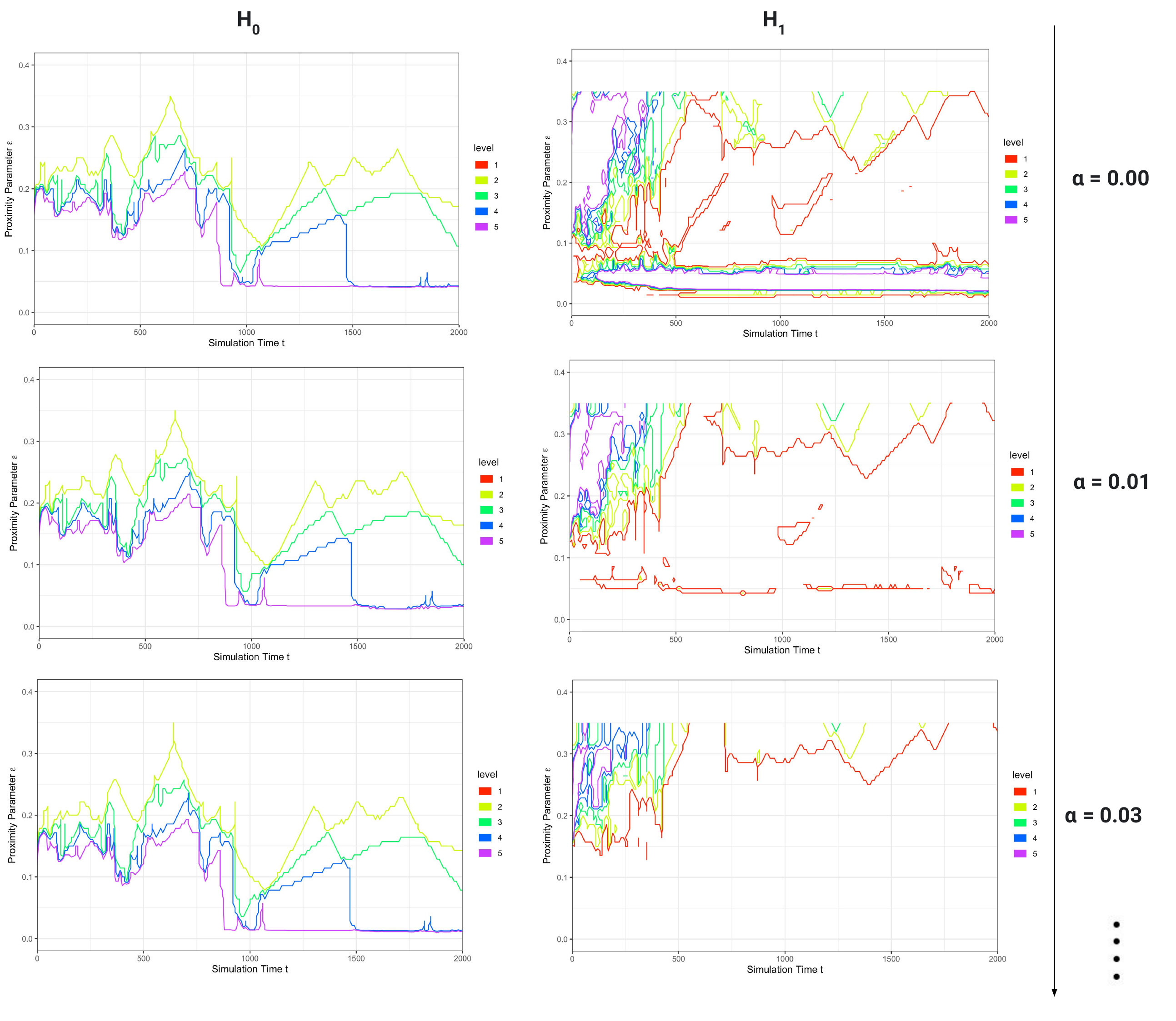}
\captionsetup{width=.9\linewidth}
\caption{An example $H_0$ and $H_1$ crocker stack for a simulation from the Viscek model with noise parameter $\eta=0.02$.
This figure shows the shifts of Betti curves in $H_0$ and $H_1$ as smoothing parameter $\alpha$ increases from $0$ to $0.01$ and $0.03$.}
\label{fig:pseudo-crockerstackpanel}
\end{figure}

Figure~\ref{fig:crockerPlot} is an example of an $H_0$ crocker plot for a simulation with noise parameter $\eta=0.02$.
We display only the contours of Betti numbers $\beta_0 \leq 6$, interpreting larger contours as noise.
In Figure~\ref{fig:crockerPlot}, notice the large region with two connected components, namely $\beta_0=2$, over the time range from approximately $t=1100$ to 2000.
This can be interpreted as two connected components for a wide range of both scale parameter $\varepsilon$ and simulation time $t$.
By looking only at the $0$-smoothed crocker plot, there is \emph{a priori} no guarantee that these are the \emph{same} components as scale $\varepsilon$ varies.
However, since the crocker stack contains enough information to recover the persistent homology barcodes, one can confirm these are the same connected components as $\varepsilon$ varies by considering the later smoothings $\alpha>0$ in the crocker stack.
To verify these are the same connected components also as time $t$ varies, one would instead want to look at the vineyard representation.

Figure~\ref{fig:pseudo-crockerstackpanel} illustrates a stack of $\alpha$-smoothed $H_0$ and $H_1$ crocker plots for the same simulation, with the $\alpha$ values $0$, $0.01$, and $0.03$ shown.
In $H_0$, the Betti curves translate down as $\alpha$ increases; by contrast, in $H_1$, the Betti curves morph shapes as $\alpha$ increases.

We will cluster different simulations based on their crocker representations, including crocker plots, single $\alpha$-smoothed crocker plots, and a crocker stack, and we compare the clustering accuracies in Section~\ref{ssec:Findings}.

\subsubsection{Experiments}\label{ss:experiments}

We create four different experiments of increasing levels of clustering difficulty by considering different collections of noise values $\eta$ that we will try to predict from simulated data.
\begin{itemize}
\item Experiment~1.
Five $\eta$ values: $\eta = 0.01$, $0.5$, $1$, $1.5$, $2$.
\item Experiment~2.
Three $\eta$ values: $\eta = 0.01$, $0.1$, $1$.
\item Experiment~3.
Six $\eta$ values: $\eta = 0.01$, $0.02$, $0.19$, $0.2$, $1.99$, $2$.
\item Experiment~4.
Fifteen $\eta$ values: $\eta = 0.01$, $0.02$, $0.03$, $0.05$, $0.1$, $0.19$, $0.2$, $0.21$, $0.3$, $0.5$, $1$, $1.5$, $1.9$, $1.99$, $2$.
\end{itemize}
Note that the more similar $\eta$ is, the more difficult it will be to cluster the data correctly.
As a thought experiment, consider two simulations with $\eta=0.01$ and $0.02$, respectively.
Both simulations will very quickly produce an aligned group, and once the groups are aligned, they are nearly indistinguishable.
The most distinguishing data is during the transient
times, but due to the quick equilibration of the system, there is relatively little transient data.

In all four experiments, we use time step equal to 10 for the crockers; additionally, we compute the crocker plot for time step equal to 40 in Experiment~2 to see how subsampling time affects accuracy.
When computing the order parameter, we instead use time step equal to 1 (in order to get the ``best" result the order parameter can provide); in Experiment~2, we also compute the order parameter with time step equal to 10 and 40 for comparison purposes.

\subsubsection{Distance matrices}

In each experiment, we vectorize the order parameter time series and crocker representations for each simulation, calculate the pairwise Euclidean distance between those vectors, and summarize the distances in a pairwise distance matrix.
To vectorize the crocker representations, we take the crocker matrix (or $\alpha$-smoothed crocker matrix) and concatenate the rows, and for the crocker stack, we further concatenate each level of $\alpha$.
Images of distance matrices for the order parameter and $H_{0,1}$ crocker plot (where the $H_0$ and $H_1$ vectorized crocker plots have been concatenated) are shown in Figures~\ref{fig:distME1}--\ref{fig:distME4} for each of the four experiments, respectively. The distance matrices for crocker stacks are visually similar to those for crocker plots and thus will not be shown here.
In each of the four pairs of comparisons, the crocker distance matrix is more structured than the order parameter distance matrix.
This is consistent with the higher clustering accuracy based on crocker plots compared to order parameters, which we discuss in Section~\ref{ssec:Findings}.
That is, distances between simulations arising from the same (or similar) noise parameter(s) typically have a smaller distance than those from different noise parameters, resulting in a block structure in the crocker distance matrix that is not as apparent in the order parameter distance matrix.

\begin{figure}[ht]
\centering
\includegraphics[width=5in]{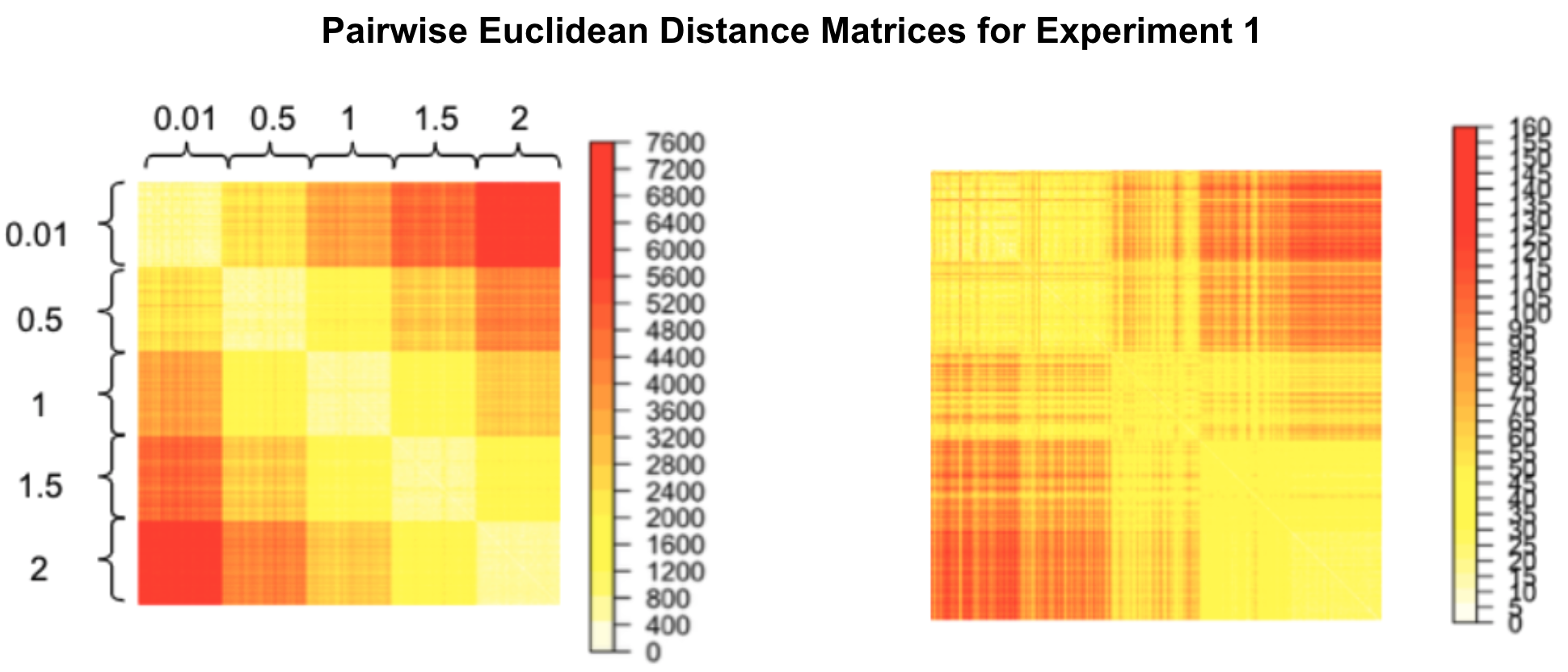}
\captionsetup{width=.9\linewidth}
\caption{
The color scale corresponds to values in the distance matrix;
red means larger distances, and yellow means smaller distances.
The 100 simulations of each noise parameter $\eta=0.01,0.5,1,1.5,2$ are listed in order and annotated in the left matrix.
The $H_{0,1}$ crocker distance matrix is more structured (Left) than the order parameter distance matrix (Right).}
\label{fig:distME1}
\end{figure}
\begin{figure}[ht]
\centering
\includegraphics[width=5in]{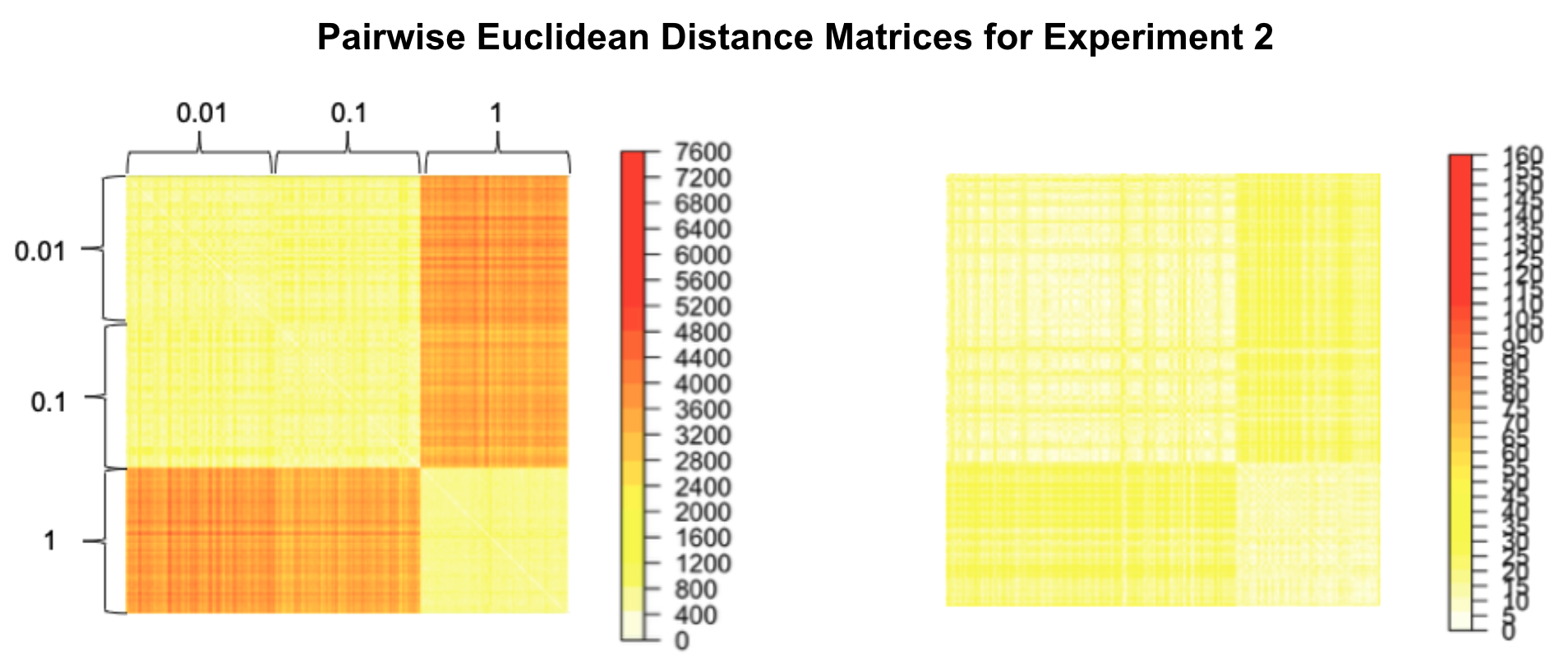}
\captionsetup{width=.9\linewidth}
\caption{
The $H_{0,1}$ crocker distance matrix is more structured (Left) than the order parameter distance matrix (Right).}
\label{fig:distME2}
\end{figure}

\begin{figure}[ht]
\centering
\includegraphics[width=5in]{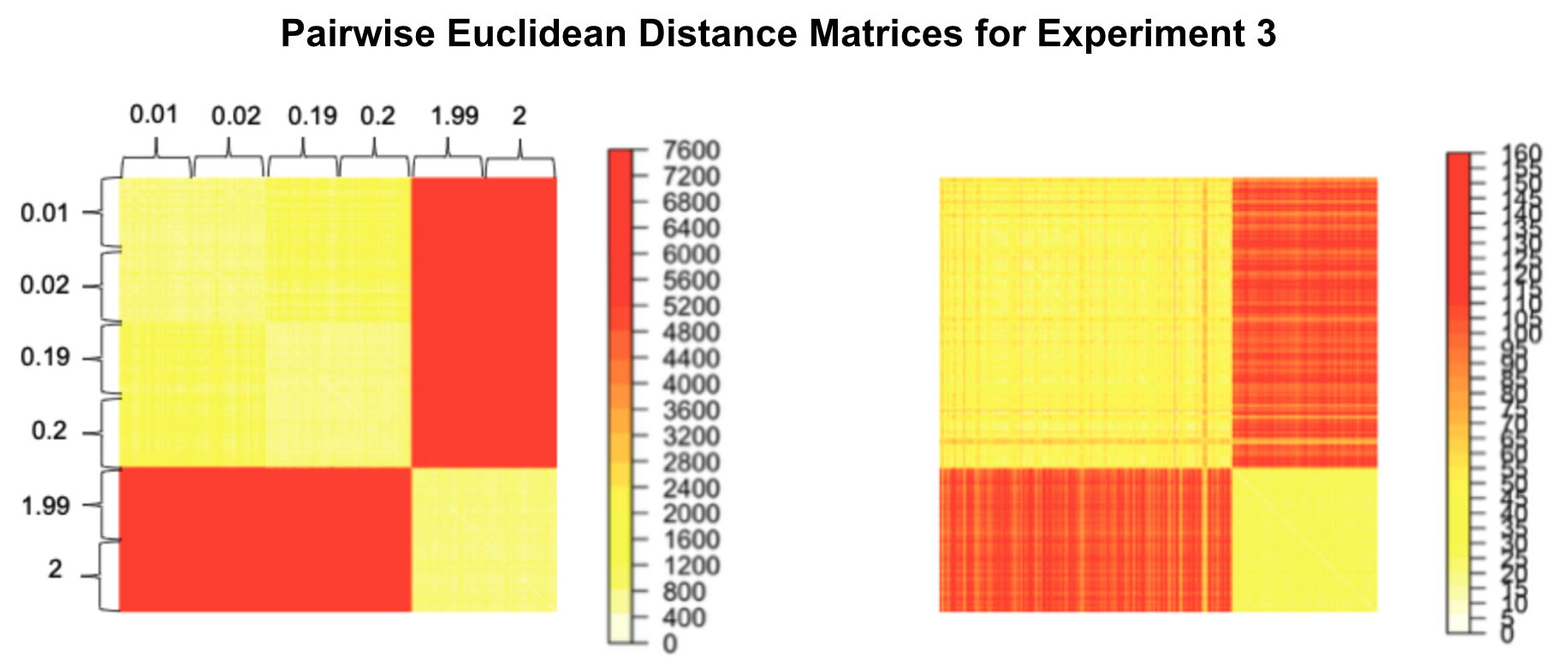}
\captionsetup{width=.9\linewidth}
\caption{
The $H_{0,1}$ crocker distance matrix is more structured (Left) than the order parameter distance matrix (Right).}
\label{fig:distME3}
\end{figure}

\begin{figure}[ht]
\centering
\includegraphics[width=5in]{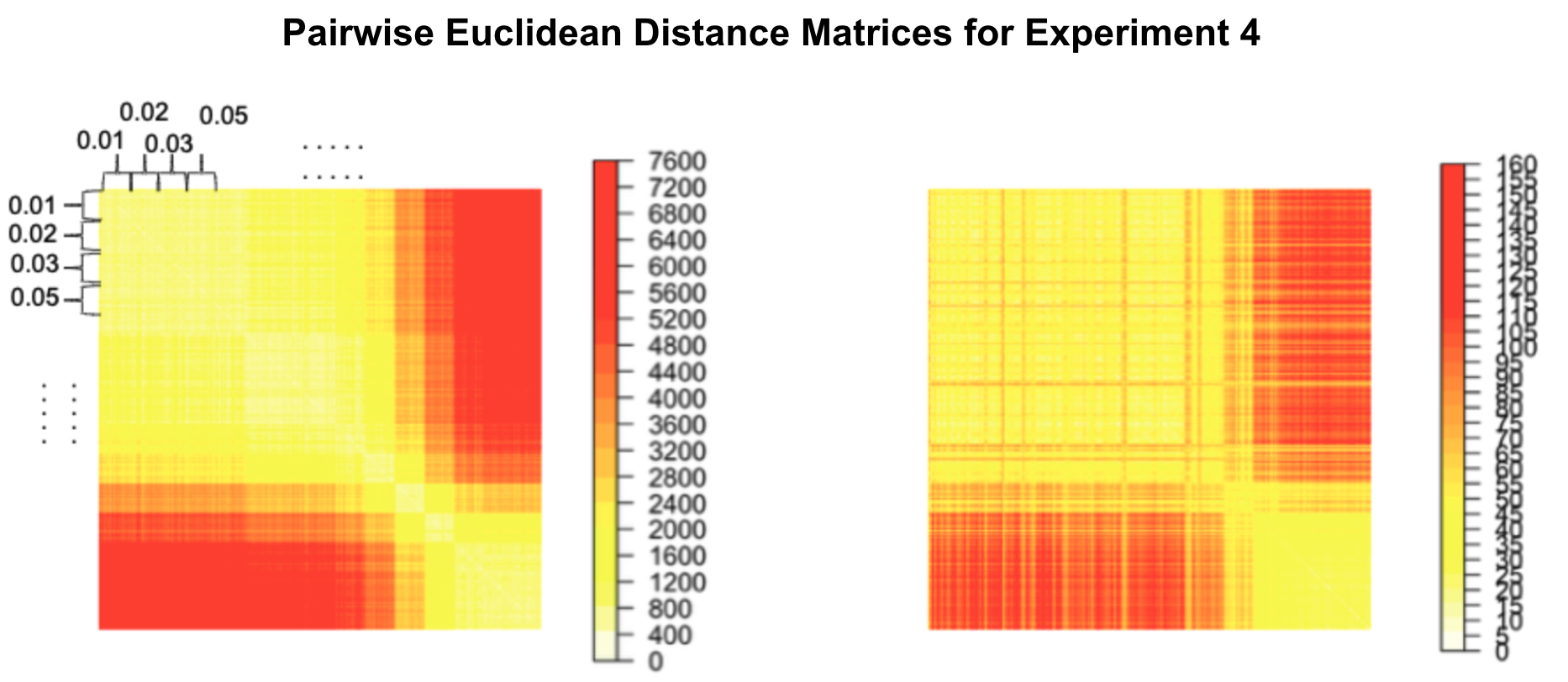}
\captionsetup{width=.9\linewidth}
\caption{
The $H_{0,1}$ crocker distance matrix is more structured (Left) than the order parameter distance matrix (Right).}
\label{fig:distME4}
\end{figure}

\subsubsection{Clustering method: $K$-medoids}
We use the \emph{$K$-medoids} algorithm to cluster the simulations from each experiment~\cite{kaufman1987clustering,park2009simple}.
This algorithm minimizes the sum of pairwise dissimilarities between data points by searching for $K$ objects from the data set, called \emph{medoids}, and then partitioning the remaining observations to their closest medoid, resulting in $K$ clusters.
The medoid is the most centrally located object of each cluster and is, in fact, an observation from the data set.
Specifically, a medoid will correspond to a particular simulation in our clustering experiments which we can trace back to the underlying noise parameter $\eta$ of the simulation.
The $K$-medoids algorithm is often more robust than the ubiquitous $K$-means algorithm, which minimizes the sum of squared Euclidean distance and thus, is more sensitive to outliers~\cite{arora}.
Another benefit of $K$-medoids is that it can take as input a distance matrix, rather than a set of feature vectors, which will be important in Section~\ref{ssec:DistanceOnVineyards} where we compute a distance on stacked sets of persistence diagrams which do not have vector representations.

We use the \emph{Partitioning Around Medoids (PAM)} algorithm in the R package ``cluster"~\cite{R-cluster} to perform the $K$-medoids algorithm, setting $K$ to be equal to the number of distinct noise parameter values $\eta$ in each experiment.
For instance, $K=5$ in Experiment~1, and $K=15$ in Experiment~4.
We compute the classification accuracy as the percentage of simulations found to be in a cluster whose medoid comes from the same noise parameter value.
Simulations partitioned into a cluster whose medoid does not come from the same noise parameter are misclassified.

\subsubsection{PCA vs.\ non-PCA}\label{ssec:PCAnonPCA}

The vectorized version of the alignment order parameter has 2001 time values ($t=0$ to 2000), and accordingly the time series are of dimension 2001. 
On the other hand, a crocker plot, which has been downsampled by a factor of 10, has 201 time values and 50 $\varepsilon$ values, resulting in a 10050-dimensional vectorized crocker for $H_0$ and $H_1$, and resulting in a 20100-dimensional vectorized crocker for the concatenation $H_{0,1}$.
It might not be ``fair" to directly compare the clustering results based on the different kinds of vectors of vastly differing dimensions.
To address this problem, we reduce the dimension of each vector to three using \emph{Principal Component Analysis (PCA)}~\cite{PCAJolliffe1986}.
This is reasonable as the first three principal components of each kind of feature vector capture about 90\% of the variance.
After reducing both the order parameter and crocker vectors for each simulation to 3-dimensional vectors, we create distance matrices based on the PCA-reduced vectors for each experiment, and then compare the PCA clustering accuracy with non-PCA clustering accuracy.

\subsection{Key findings} \label{ssec:Findings}

\subsubsection{Results summary}\label{ssec:ResultsSummary}

Table~\ref{table:AccuracySummary} shows a summary of the accuracies from clustering using $K$-medoids on each of the four experiments with different input feature vectors (order parameters, crocker plots, and crocker stacks) using both the full representations (non-PCA) and dimensionality reduced versions (PCA, italicized).
When we refer to the crocker stacks in the table, we are considering the stack of the 18 $\alpha$-smoothed crocker plots discussed in Section~\ref{ssec:CrockerPlot}, where we vectorize the crocker for each $\alpha$ and then concatenate.
In Section~\ref{ssec:SingleAlpha}, we consider the clustering accuracy of single $\alpha$-smoothed crocker plots in comparison with this stacked version.
The crocker representations consider the homology dimensions $H_0$ and $H_1$ as well as the concatenation $H_{0,1}$ of the two.

\begin{table}[ht]
\caption{Summary of the clustering accuracy on four different experiments (abbreviated Exp.) with three different feature vectors: order parameters, crocker plots, and crocker stacks.
For crocker plots and crocker stacks, we distinguish different homological dimensions: $H_{0,1}$, $H_0$, and $H_1$.
The top accuracy scores of each column are bolded.
This table summarizes results with time step 1 for order parameters and time step 10 for crocker representations.
Results with other time steps are discussed in Section~\ref{ssec:ResultsSummary}.
Clustering results with feature vectors that have been reduced to 3 dimensionsby PCA are shown in italics, while full feature vectors are not italicized.}
\label{table:4ExperimentsAccuracyPCAnonPCA.B}
\renewcommand{\arraystretch}{1.3}
\centering
\scalebox{0.95}{\centering
\begin{tabular}{|>{\bfseries}l||c|>{\em}c|c|>{\em}c|c|>{\em}c|c|>{\em}c|}  

\hline   
& \multicolumn{2}{c|}{\textbf{Exp.\ 1}}  & \multicolumn{2}{c|}{\textbf{Exp.\ 2}}  & \multicolumn{2}{c|}{\textbf{Exp.\ 3}} & \multicolumn{2}{c|}{\textbf{Exp.\ 4}} \\ \hline \hline
Order Parameters & 0.63 & 0.51 & 0.61 & 0.59  & 0.35 &  0.35  & 0.21 & 0.17 \\ \hline \hline

Crocker Plots, $H_{0,1}$ & \textbf{1.00} & \textbf{1.00} & 0.67  & 0.67 & 0.44 & 0.43  & \textbf{0.42}  & \textbf{0.43} \\ \hline
Crocker Plots, $H_0$ & \textbf{1.00} & \textbf{1.00} & 0.67  & \textbf{0.77} & 0.45 & 0.43  & 0.39  & \textbf{0.43} \\ \hline
Crocker Plots, $H_1$ & 0.98 & 0.99 & \textbf{0.71}  & 0.67 & 0.36 & 0.35  & 0.37  & 0.33 \\ \hline \hline

Crocker Stacks, $H_{0,1}$ & \textbf{1.00} & 0.98 & 0.67 & 0.67 & 0.47 & 0.38 & 0.41 & 0.35 \\ \hline 
Crocker Stacks, $H_0$ & \textbf{1.00} & \textbf{1.00} & 0.67 & 0.67 & \textbf{0.49} & \textbf{0.46} & 0.41 & 0.41 \\ \hline
Crocker Stacks, $H_1$ & 0.96 & 0.98 & 0.63 & 0.67 & 0.34 & 0.37 & 0.32 & 0.35 \\ \hline

\end{tabular}}
\label{table:AccuracySummary}
\end{table}

Crocker plot and stack clustering accuracies are higher than order parameter accuracies across all experiments.
We perform paired sample t-tests to compare the means of experiment accuracies between order parameter and crocker plot ($H_{0,1}$) feature vectors  and between order parameter and crocker stack ($H_{0,1}$) feature vectors.
The 1-tail p-value for the t-test between order parameters and crocker plots is $0.04$, and the p-value for the t-test between order parameters and crocker stacks is $0.03$, which are both significant at the significance criterion $0.05$ level.
In other words, the means of the four experiments' accuracies for both crocker plots and stacks are significantly higher than for order parameters. 
As expected, Experiment~1 accuracy is the highest compared to other experiments for both order parameters and crocker plots.
This is because the noise $\eta$ parameters are approximately evenly spaced in Experiment~1.
Experiment~4 accuracy is the lowest, since Experiment~4 has both close and very different $\eta$ parameters, which confuses classification.
Generally speaking, PCA (shown in italics in Table~\ref{table:4ExperimentsAccuracyPCAnonPCA.B}) and non-PCA accuracies are comparable across all four experiments.
It is striking to us that even though we have reduced the dimension of the data down to 3 via PCA, we observe little degradation in cluster accuracy.
The accuracies for crocker stacks are comparable with those for crocker plots.



Since $H_{0,1}$ distance matrices generally contain more information than either $H_0$ and $H_1$, they often generate higher clustering accuracies than either $H_0$ or $H_1$.
We observe two exceptions for non-dimensionality reduced crocker plots: $H_1$ accuracy for Experiment~2 is slightly higher than $H_{0,1}$ and $H_0$ accuracy; $H_0$ accuracy for Experiment~3 is slightly higher than $H_{0,1}$ and $H_0$ accuracy.
For non-dimensionality reduced croker stacks, $H_0$ accuracy is slightly higher than $H_{0,1}$ in Experiment 3.

Among order parameters, crocker plots, and crocker stacks, the stacks encode the most information, since they contain $\alpha$-smoothed crocker plots over several different $\alpha$ values.
In $H_0$, the contour lines of $\alpha$-smoothed crocker plots continuously shift down as $\alpha$ increases. 
This is because, in a Vietoris-Rips complex, all vertices are born at scale $\varepsilon=0$.
Thus, the contour lines separating different ranks in the crocker plot translate as $\alpha$ increases, but are otherwise unchanged.
In contrast, in $H_1$, the contour lines of $\alpha$-smoothed crocker plots do not continuously shift down but morph in shape.
See the links for videos of \href{https://youtu.be/nv9QAYSQTFc}{$H_0$} and \href{https://youtu.be/_SIrOYUctzY}{$H_1$}
crocker stacks from a simulation corresponding to noise parameter $\eta=0.02$ of the Viscek model.

To further detail how the clustering accuracies in Table~\ref{table:AccuracySummary} arise, we consider Experiment 3 as a case study.
Misclassifications can occur for two reasons.
First, as alluded to previously, it will be very difficult to accurately cluster simulations that have different values of $\eta$ that are all in the small-noise regime, leading to strong alignment of the system.
When the system aligns strongly and quickly, information contained in the transient state is lost.
Second, even when the system is not in the strong-alignment regime, misclassifications may occur simply when values of $\eta$ are sufficiently close together.
These difficulties are exemplified in Table~\ref{table:ConfusionMatrix}, which shows the confusion matrix of the $H_{0,1}$ crocker plots of Experiment 3 clustered using $K$-medoids, which contains simulations with $\eta$ = 0.01, 0.02, 0.19, 0.2, 1.99.
The rows represent the actual simulations corresponding to each noise parameter $\eta$ while the columns represent the parameter of the cluster medoid to which a simulation is assigned.
Even though $K=6$ clusters were formed, the six medoids correspond to simulations from only three distinct noise parameters $\eta=0.02$, $0.2$, $2$.
That is, there were two medoids from each of these three noise parameter classes selected by the algorithm, and we group clusters with the same noise parameter together.
All simulations corresponding to noise parameters $\eta=0.01$, $0.19$, $1.99$ are misclassified as there are no medoids selected from these parameter classes.
\begin{table}[ht]
\caption{Confusion matrix using $K$-medoids to cluster the $H_{0,1}$ crocker plots corresponding to simulations of Experiment 3 ($\eta$=0.01, 0.02, 0.19, 0.2, 1.99, 2).
The rows represent the actual simulations corresponding to each noise parameter $\eta$ while the columns represent the parameter of the cluster medoid to which a simulation is assigned.
Even though $K=6$ clusters were formed, the six medoids correspond to simulations from only three distinct noise parameters $\eta=0.02,0.2,2$.}
\label{table:ConfusionMatrix}
\renewcommand{\arraystretch}{1.3}
\centering
\scalebox{0.95}{\centering
\begin{tabular}{|>{\bfseries}l||c|c|c|}   
\hline 
 & \multicolumn{3}{c|}{\textbf{$K$-medoids Clusters}} \\ \hline
& \setrow{\bfseries} $\eta$ = 0.02  & \setrow{\bfseries} $\eta$ = 0.20  & \setrow{\bfseries} $\eta$ = 2.00  \\ \hline \hline
$\eta$ = 0.01 & 69 & 31 & 0  \\ \hline
$\eta$ = 0.02 & 64 & 36 & 0  \\ \hline
$\eta$ = 0.19 & 4 & 96 & 99 \\ \hline
$\eta$ = 0.2 & 1 & 99 & 0  \\ \hline
$\eta$ = 1.99 & 0& 0 & 100 \\ \hline 
$\eta$ = 2.00 & 0& 0& 100 \\ \hline
\end{tabular}
}
\end{table}
Notice that 69 of 100 simulations from class $\eta$ = 0.01 are partitioned into a cluster with a medoid coming from class $\eta=0.02$, and 64 of 100 simulations from class $\eta$ = 0.02 are partitioned into the same cluster.
This pattern is consistent for simulations with $\eta$ values of the same order of magnitude: simulations from classes $\eta=0.19$ and 0.2 are largely partitioned into a cluster with a medoid coming from class $\eta=0.2$, and simulations from classes $\eta=1.99$ and 2 are all partitioned into a cluster with a medoid coming from class $\eta=2$.
In addition, simulations from the first four classes ($\eta$ = 0.01, 0.02, 0.19, 0.2) are partitioned into either of the clusters with medoids from $\eta=0.02$ or 0.2.
Recall that values of $\eta$ that are in the alignment regime will produce data that is essentially identical apart from a very brief transient phase and hence, are difficult to distinguish.

We now consider the effect of subsampling time on classification accuracy.
The order parameter clustering accuracy for Experiment~2 with time step 1 is 0.61, and this is the same as subsampling the data by time steps of 10 and 40.
In addition, the crocker plot and crocker stack accuracy for Experiment~2 with time steps 10 and 40 are the same.
This shows that the effect of these subsamplings of time on clustering accuracy is negligible in this particular experiment.

\subsubsection{$\alpha-$smoothed crocker plots: stack vs.\ single $\alpha$'s}\label{ssec:SingleAlpha}
We compare the clustering accuracy of the $H_{0,1}$ stacked $\alpha$-smoothed crocker plots (with 18 $\alpha$ values combined) and those with a single $\alpha$ value in Table~\ref{table:SingleAlphaAccuracy}.
The crocker stack performs comparably to the individual $\alpha$-smoothed crocker plots.

As $\alpha$ increases, accuracy decreases.
This makes sense because as $\alpha$ increases, more smoothing of distinct topological features occurs, which may ignore some information necessary for parameter identification.

\begin{table}[ht]
    \caption{Summary of the clustering accuracy for the four experiments based on single $\alpha$-smoothed crocker plots in $H_{0,1}$ as input feature vectors to $K$-medoids, and the accuracy based on the crocker stack (with 18 $\alpha$ values combined).
    The top accuracy scores of each column are bolded.
Recall that when $\alpha=0$, the $\alpha$-smoothed crocker plot is equivalent to the standard crocker plot of~\cite{topaz2015topological}.}
\label{table:SingleAlphaAccuracy}
\renewcommand{\arraystretch}{1.3}
\centering
\scalebox{0.95}{\centering
\begin{tabular}{|>{\bfseries}l||c|c|c|c|}
\hline   
 & \setrow{\bfseries} Exp.\ 1  & \setrow{\bfseries} Exp.\ 2  & \setrow{\bfseries} Exp.\ 3 & \setrow{\bfseries} Exp.\ 4  \\ \hline \hline
stack & \bf{1.00} & 0.67 & 0.47 & 0.41 \\ \hline
$\alpha$ = 0.00 & \bf{1.00} & 0.67  & 0.44 & \bf{0.42}  \\ \hline
$\alpha$ = 0.01 & \bf{1.00} & 0.67 & 0.46 & 0.40 \\ \hline
$\alpha$ = 0.03 & \bf{1.00} & 0.67 & \bf{0.49} & 0.40 \\ \hline
$\alpha$ = 0.05 & \bf{1.00} & 0.58 & 0.44 & 0.38 \\ \hline
$\alpha$ = 0.08 & 0.99 & 0.67 & 0.39 & 0.36 \\ \hline 
$\alpha$ = 0.11 & 0.97 & 0.67 & 0.33 & 0.35 \\ \hline
$\alpha$ = 0.13 & 0.92 & \bf{0.73} & 0.41 & 0.28 \\ \hline
$\alpha$ = 0.17 & 0.73 & 0.66 & 0.40 & 0.21 \\ \hline
\end{tabular}
}
\end{table}

\subsubsection{Distances on stacked sets of persistence diagrams} \label{ssec:DistanceOnVineyards}

We now want to compare the clustering accuracy of the crocker plot representations to the stacked set of persistence diagrams representation.
Recall as introduced in Section~\ref{ssec:vineyards}, the latter can be thought of as time-varying persistence diagrams, since they contain the births and deaths of topological features over the scale parameter $\varepsilon$ and time $t$.
Suppose two time-varying metric spaces $\textbf{X}$ and $\textbf{Y}$ have persistence diagrams $PH(\vrp{X_t})$ and $PH(\vrp{Y_t})$ for all $t$.
One way to compute distance between these stacked sets of persistence diagrams is to compute the bottleneck distance between each pair of persistence diagrams $PH(\vrp{X_t})$ and $PH(\vrp{Y_t})$ at each time $t$ and then take the supremum of the bottleneck distances over all time values.
This \emph{supremum bottleneck distance} (which we henceforth refer to as the bottleneck distance between stacked sets of persistence diagrams) is defined as follows:

\[d_b^\infty(\ph(\vrp{\bX}),\ph(\vrp{\bY}))=\sup_t d_b\Bigl(\ph(\vrp{X_t}),\ph(\vrp{Y_t})\Bigr).\]

As this computation is extremely expensive, we only compute this bottleneck distance for Experiment~2, which entails 300 simulations.
The $K$-medoids clustering accuracies based on this bottleneck distance matrix, along with corresponding accuracies of the Euclidean distances of the order parameters, crocker plots, and crocker stacks, are shown in Table \ref{table:BottleneckAccuracy}.

\begin{table}[ht]
\captionsetup{width=.9\linewidth}
\caption{Comparison of the $K$-medoids clustering accuracy of Experiment~2 of the Euclidean distance on order parameters, crocker plots, and crocker stacks, as well as the bottleneck distance on the stacked set of persistence diagrams.
All topological representations compute homology in dimensions 0 and 1, denoted $H_0$ and $H_1$.
The top accuracy scores of each column are bolded.
The parentheses on the order parameter row indicate that the same computation is performed in both columns since order parameters do not incorporate homology dimensions.
}
\label{table:BottleneckAccuracy}
\renewcommand{\arraystretch}{1.3}
\centering
\scalebox{0.95}{\centering
\begin{tabular}{|>{\bfseries}l||c|c|}   
\hline   
 & \setrow{\bfseries} $H_0$& \setrow{\bfseries} $H_1$  \\ \hline \hline
Order Parameters & (0.61) & (0.61) \\ \hline
Crocker Plots & \bf{0.67} & \bf{0.71} \\ \hline
Crocker Stacks & \bf{0.67} & 0.63 \\ \hline
Stacked Persistence Diagrams & \bf{0.67} & 0.49 \\ \hline
\end{tabular}}
\end{table}

As a stacked set of persistence diagrams is not a vector representation of the topological features, it cannot be directly fed into a machine learning algorithm as a feature vector.
However, the $K$-medoids clustering algorithm can take as input a distance matrix rather than a set of feature vectors.
The pairwise bottleneck distance between stacked sets of persistence diagrams corresponding to simulations serves as our input for $K$-medoids.
As shown in Table~\ref{table:BottleneckAccuracy},  clustering with the bottleneck distance yields higher accuracy in $H_0$ but lower in $H_1$ as compared to order parameters.
While we compute the order parameter at every time step in order to provide the ``best" possible results, we downsample in time with a step size of 10 for the bottleneck computation due to computational complexity.
The clustering accuracy for stacked persistence diagrams is either the same as or lower than crocker plots.

We now contrast the computational complexities for computing bottleneck distances of stacked persistence diagrams and Euclidean distances of crocker plots or stacks.
Both processes start with the persistent homology data from our simulations over time.
While we can directly compute the bottleneck distance between simulations from the interval data at a fixed time and then find the supremum over all times, in the crocker representations, we first need to transform the interval data to crocker plots or stacks, vectorize the crocker representation, and then compute the pairwise Euclidean distance between simulations.
Even though there is a transformational step to convert the interval data into the crocker representations, the computational time for bottleneck distance matrices is roughly four orders of magnitude larger than for Euclidean distance matrices, \emph{including} the transformation to crockers.
Since clustering with stacked persistence diagrams does not yield higher clustering accuracy for this experiment and is far more time-intensive than the crocker representations, we posit that crockers may serve as a better means for parameter identification.
Further, as crockers are vector representations, they are more amenable to a host of machine learning tools.

This concludes our experimental analysis of parameter identification of the Vicsek model using order parameters, crocker plots, crocker stacks, and stacked sets of persistence diagrams.

\section{Distances between metric spaces and persistence modules}\label{sec:distances}

We now survey a variety of notions of distances between metric spaces and persistence modules, which will be useful for describing the continuity properties of crocker stacks in Section~\ref{sec:continuity}.

\subsection{The Hausdorff distance}

In order to introduce the stability of persistent homology, we will need some notion of distance on metric spaces.
The Hausdorff distance~\cite{munkres1975topology} measures the distance between metric spaces $X$ and $Y$ that are ``aligned".
More precisely, we mean that $X$ and $Y$ are subsets of a larger metric space $(Z,d)$ that contains both $X$ and $Y$ as submetric spaces.
For $X\subseteq Z$ and $\delta>0$, let 
$X^\delta:=\{z\in Z~|~d(z,x)\le \delta\text{ for some }x\in X\}$
denote the \emph{$\delta$-offset} of $X$ in $Z$.

\begin{definition}
If $X$ and $Y$ are two subsets of a metric space $Z$, then the \emph{Hausdorff distance} between $X$ and $Y$ is $d_H^Z(X,Y)=\inf\{\delta>0~|~X\subseteq Y^\delta\text{ and }Y\subseteq X^\delta\}$, which is equivalent to 
\[d_H^Z(X,Y)=\max \left\{ \sup_{x\in X} \inf_{y \in Y} d(x,y),\ \sup_{y \in Y} \inf_{x \in X} d(x,y) \right\}.\] 
\end{definition}
An illustration of the Hausdorff distance is shown in Figure~\ref{fig:Hausdorff}.

\begin{figure}[ht]
\centering
\includegraphics[width=1.8in]{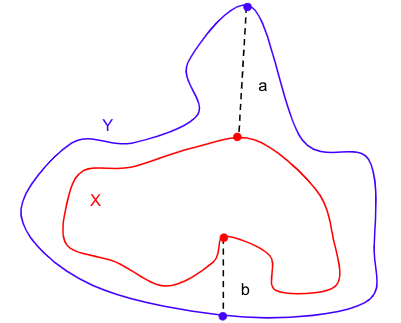}
\captionsetup{width=.9\linewidth}
\caption{Let $Z=\mathbb{R}^2$ with the Euclidean metric.
The red curve represents subset $X$ of $Z$, and the blue curve represents subset $Y$ of $Z$.
To compute the Hausdorff distance between $X$ and $Y$, we first take the supremum over all points in $Y$ of the distance to the closest point in $X$.
In this figure, the distance is $a = \sup_{y \in Y} \inf_{x \in X} d(x,y)$.
Then, we do the same for the supremum over all points in $X$ of the distance to the closest point in $Y$, as shown by $b = \sup_{x\in X} \inf_{y \in Y} d(x,y)$.
Finally, we take the maximum of the two suprema, $d_H^Z(X,Y)=\max\{a,b\}=a$.}
\label{fig:Hausdorff}
\end{figure}

The Hausdorff distance is an extended pseudo-metric on the subsets of $Z$; two sets have Hausdorff distance zero if and only if they have the same closure, and the Hausdorff distance between unbounded sets can be infinite.
When restricted to the set of all non-empty compact subsets of $Z$, the Hausdorff distance is in fact a metric.
We often write $d_H$, instead of $d_H^Z$, in order to simplify notation when space $Z$ is clear.

\subsection{Gromov--Hausdorff distance}
\label{ssec:gh}

The Gromov--Hausdorff distance~\cite{BuragoBuragoIvanov} allows us to define a notion of distance between two metric spaces that are not aligned in any sense.

\begin{definition}
The \emph{Gromov--Hausdorff} distance between two metric spaces $X$ and $Y$ is
\[d_\gh(X,Y)=\inf_{Z,f,g}d_H^Z(f(X),g(Y)),\]
where the infimum is taken over all possible metric spaces $Z$ and isometric embeddings $f\colon X\to Z$ and $g\colon Y\to Z$.
\end{definition}

The Gromov--Hausdorff distance is an extended pseudo-metric on metric spaces (non-isometric spaces, such as the rationals and the reals, can have Gromov--Hausdorff distance zero).
The Gromov--Hausdorff distance is a metric when restricted to the quotient space of compact metric spaces under the equivalence relation of isometry.

\subsection{The stability of persistent homology}
\label{ssec:stability-ph}

By the stability of persistent homology~\cite[Theorem~5.2]{ChazalDeSilvaOudot2014}, if $X$ and $Y$ are compact metric spaces, then the bottleneck distance satisfies
\[d_b(\ph(\vrp{X}),\ph(\vrp{Y})) \le 2d_\gh(X,Y).\]
An analogous bound is true if Vietoris--Rips complexes are replaced with \v{C}ech complexes.

In many applications, spaces $X$ and $Y$ are usually already embedded in the same space, in which one can use the bound $d_\gh \leq d_{H}$ in order to get a lower bound on the Hausdorff distance between these particular embeddings.
The stability theorem states that if two metric spaces are close, then the bottleneck distance between their persistence diagrams (using the Vietoris--Rips or \v{C}ech complex to construct the filtration) will also be close.
Stability is a useful property as it allows for small perturbations of the inputs, and as such, stability of persistence modules has led to their effectiveness in data analysis.
One of our motivations for defining crocker stacks is their analogous continuity properties.

\subsection{The interleaving distance}\label{ssec:interleaving}

A closely related notion to the bottleneck distance between persistence diagrams associated to persistence modules (Section~\ref{ssec:bottleneck}) is that of $\delta$-interleaving~\cite{oudot2015persistence}.

\begin{definition}
\label{def:interleaving}
For two persistence modules $V$ and $W$, a \emph{$\delta$-interleaving} (for $\delta\ge 0$) is given by two families of linear maps ($\phi_i:V^i \to W^{i+\delta}$) and ($\psi_i:W^i \to V^{i+\delta}$) such that the following diagrams commute for all $i \leq j$:
\begin{center}
\begin{tikzcd}
V^i \arrow[r]   \arrow[rd, "\phi_i"] &   V^j \arrow[rd, "\phi_j"] & \\
  & W^{i+\delta} \arrow[r] & W^{j+\delta} 
\end{tikzcd}
\hspace{.5cm}
\begin{tikzcd}
& V^{i+\delta} \arrow[r]    &   V^{j+\delta}  \\
  W^{i} \arrow[r] \arrow[ru, "\psi_i"]& W^{j} \arrow[ru, "\psi_j"] &
\end{tikzcd}
\newline
\begin{tikzcd}
V^i \arrow[rr]   \arrow[rd, "\phi_i"] & &   V^{i+2\delta} \\
  & W^{i+\delta} \arrow[ru, "\psi_{i+\delta}"] 
\end{tikzcd}
\hspace{.5cm}
\begin{tikzcd}
& V^{i+\delta} \arrow[rd, "\phi_{i+\delta}"]   \\
  W^{i} \arrow[rr] \arrow[ru, "\psi_i"] & & W^{i+2\delta}
\end{tikzcd}
\end{center}
Note that a $0$-interleaving between two persistence modules is nothing more than an isomorphism between them.
The \emph{interleaving distance} between persistence modules $V$ and $W$ is defined as
\[d_I(V, W) = \inf \{ \delta \geq 0~|~\text{there is a } \delta\text{-interleaving between } V \text{ and } W\}.\]
\end{definition}
\noindent Roughly speaking, one should think of the interleaving distance between two persistence modules as a measure of how far they are from being isomorphic.

In this paper, we do not directly use the interleaving distance on persistence modules.
However, by the isometry theorem of Lesnick~\cite{lesnick2015theory}, the interleaving distance is equal to the bottleneck distance, namely $d_I=d_b$.
As such, whenever we invoke the bottleneck distance, we could alternatively invoke the interleaving distance on persistence modules.
The morphisms induced on the persistence modules by interleaving will allow us to prove Lemma~\ref{lem:interleaving-fixed-t}, which is helpful for describing the continuity of crocker stacks.

\subsection{The rank invariant}
\label{sec:rank}

Let $V$ be a persistence module.
We recall that the collection of all natural numbers
$\rank(V(\varepsilon) \to V(\varepsilon'))$
for all choices of $\varepsilon \le \varepsilon'$ is called the \emph{rank invariant}.
The rank invariant is equivalent to the peristence barcode, in the sense that it is possible to obtain either one from the other~\cite[Theorem~12]{carlsson2009theory}.
An interesting historical comment is that persistent homology of a finite set of points $X$ was first defined as the collection of ranks of all maps of the form $H(\vr{X,\varepsilon-\alpha}) \to H(\vr{X,\varepsilon+\alpha})$, as opposed to as a persistence module, barcode, or diagram.
See for example the definition on page 151 of~\cite{EdelsbrunnerHarer}, where their $i$ is $\varepsilon-\alpha$, and where their $j$ is $\varepsilon+\alpha$.

Given a persistence module $V$, we can encode the rank invariant for $V$ as a function $g_V\colon [0,\infty)\times[0,\infty)\to\N$, where $g_V(\varepsilon,\alpha)=\rank(V(\varepsilon-\alpha) \to V(\varepsilon+\alpha))$.
The function $g_V$ is a function on a 2D domain, where the two dimensions are scale $\varepsilon$ and persistence parameter $\alpha$.

We describe a connection between one time-slice of a crocker stack and the rank invariant.
Let $\bV$ be a time-varying persistence module, and consider a fixed time $t$.
The 2D representation $g_{V_t}\colon [0,\infty)\times[0,\infty)\to\N$, for a fixed time $t$, is a cross-sectional slice of the 3D crocker stack.
Therefore, a crocker stack encodes the rank invariant, and therefore the persistence barcode, of the persistence module $V_t$ at each time $t$.
Persistence lanscapes can be interpreted as a sequence of rank functions, and therefore persistence landscapes~\cite{Bubenik2015} are also closely related to these cross-sectional slices (fixing time) in a crocker stack.

\subsection{Relationship of rank invariant and bottleneck distance} 

We now describe a relationship between the rank invariant and the bottleneck distance.
We consider two persistence modules $V$ and $W$ which decompose into a finite number of intervals, equipped with rank invariants $g_V$ and $g_W$.

\begin{lemma}\label{lem:interleaving-fixed-t}
If $d_b(V,W)\le\delta$, then for all $\varepsilon$ and $\alpha$ we have
\begin{itemize}
\item $g_V(\varepsilon,\alpha+\delta)\le g_W(\varepsilon,\alpha)$,
and
\item $g_W(\varepsilon,\alpha+\delta)\le g_V(\varepsilon,\alpha)$.
\end{itemize}
\end{lemma}

\begin{proof}

By the equivalence between the bottleneck and interleaving distances of persistence modules~\cite{lesnick2015theory,oudot2015persistence}, since $d_b(V,W)\le\delta$, there exist morphisms $\phi_\varepsilon\colon V(\varepsilon)\to W(\varepsilon+\delta)$ and $\psi_\varepsilon\colon W(\varepsilon)\to V(\varepsilon+\delta)$ for all $\varepsilon$, along with the following commutative diagrams.
\begin{center}
\begin{tikzcd}
V(\varepsilon-\alpha-\delta) \arrow[rrr] \arrow[rd] &  &  & V(\varepsilon+\alpha+\delta) \\
 & W(\varepsilon-\alpha) \arrow[r] & W(\varepsilon+\alpha) \arrow[ru] & 
\end{tikzcd}

\begin{tikzcd}
W(\varepsilon-\alpha-\delta) \arrow[rrr] \arrow[rd] &  &  & W(\varepsilon+\alpha+\delta) \\
 & V(\varepsilon-\alpha) \arrow[r] & V(\varepsilon+\alpha) \arrow[ru] & 
\end{tikzcd}
\end{center}
Indeed, note that the trapezoids above can be obtained by gluing together a parallelogram and a triangle from Definition~\ref{def:interleaving}.
Define $V_{i}^{j}$ to be the map from $V(i)$ to $V(j)$.
By the first commutative diagram, $V_{\varepsilon-\alpha-\delta}^{\varepsilon+\alpha+\delta} $ = $\psi_{\varepsilon+\alpha} \circ W_{\varepsilon-\alpha}^{\varepsilon+\alpha} \circ \phi_{\varepsilon-\alpha-\delta}$.
Since the rank of a composition of linear transformations is at most the minimum rank of the transformations, we have
\[\rank(V_{\varepsilon-\alpha-\delta}^{\varepsilon+\alpha+\delta}) \leq \min\{\rank(\psi_{\varepsilon+\alpha}), \rank(W_{\varepsilon-\alpha}^{\varepsilon+\alpha}), \rank(\phi_{\varepsilon-\alpha-\delta})\}.\]
Thus, $\rank(V_{\varepsilon-\alpha-\delta}^{\varepsilon+\alpha+\delta}) \leq \rank(W_{\varepsilon-\alpha}^{\varepsilon+\alpha})$, and so $g_V(\varepsilon,\alpha+\delta)\le g_W(\varepsilon,\alpha)$.

A similar argument for the second diagram shows $g_W(\varepsilon,\alpha+\delta)\le g_V(\varepsilon,\alpha)$.
\end{proof}

We remark that the infimum $\delta\ge 0$ such that $g_V(\varepsilon,\alpha+\delta)\le g_W(\varepsilon,\alpha)$ and $g_W(\varepsilon,\alpha+\delta)\le g_V(\varepsilon,\alpha)$ for all $\varepsilon$ and $\alpha$ is the \emph{erosion distance}~\cite{Patel2018,Edelsbrunner2018,Dey2018,puuska2017erosion} between the two persistence modules $V$ and $W$.
Lemma~\ref{lem:interleaving-fixed-t} shows that if two persistence modules are close in the bottleneck distance ($\le \delta$), then their rank invariants are also close in a sense encapsulated by the erosion distance ($\le \delta$).
Could a bound in the reverse direction also be true?

\begin{question}\label{ques:interleaving-fixed-t-iff}
Is the bottleneck distance $d_b(V,W)$ equal to the erosion distance between $V$ and $W$, \emph{i.e.}, the infimum over all $\delta\ge 0$ such that 
\begin{itemize}
\item $g_V(\varepsilon,\alpha+\delta)\le g_W(\varepsilon,\alpha)$,
and
\item $g_W(\varepsilon,\alpha+\delta)\le g_V(\varepsilon,\alpha)$
\end{itemize}
for all $\varepsilon$ and $\alpha$?
\end{question}
The answer is no, as shown by an example proposed by Amit Patel and Brittany Terese Fasy.

\begin{example}
\begin{figure}[ht]
\centering
\includegraphics[width=2in]{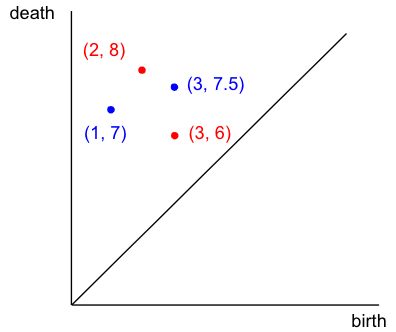}
\captionsetup{width=.9\linewidth}
\caption{Consider the persistence diagrams $\dgm_k(V) = \{ (3,6), (2,8) \}$ and $\dgm_k(W) = \{ (1,7), (3,7.5) \}$.
The bottleneck distance between the two persistence diagrams is 1.5, while the erosion distance is 1.
}
\label{fig:counterexample}
\end{figure}
Consider persistence diagrams 
\[\dgm_k(V) = \{ (3,6), (2,8) \} \text{ and } \dgm_k(W) = \{ (1,7), (3,7.5) \},\]
as shown in Figure \ref{fig:counterexample}.
The bottleneck distance between the two persistence diagrams is 1.5.
There are two natural bijections between the two persistence modules, not including matching points with the diagonal (which in this example leads to higher costs).
In bijection (i), we match $(3,6)$ in $\dgm_k(V)$ with $(3,7.5)$ in $\dgm_k(W)$ and $(2,8)$ in $\dgm_k(V)$ with $(1,7)$ in $\dgm_k(W)$.
The $L_{\infty}$ distance between matched points is 1.5.
In bijection (ii), we match $(3,6)$ in $\dgm_k(V)$ with $(1,7)$ in $\dgm_k(W)$ and $(2,8)$ in $\dgm_k(V)$ with $(3,7.5)$ in $\dgm_k(W)$.
The $L_{\infty}$ distance between matched points is 2.
The bottleneck distance is the infimum $L_{\infty}$ distance between matched points over all bijections, which is 1.5.

The erosion distance is the infimum $\delta\ge 0$ such that for all $\varepsilon$ and $\alpha$, we have $g_V(\varepsilon,\alpha+\delta)\le g_W(\varepsilon,\alpha)$,
and $g_W(\varepsilon,\alpha+\delta)\le g_V(\varepsilon,\alpha)$.
For this example, $\delta=1$, for the following reasons:
The maximum interval over which $V$ has rank two is $[3,6]$, whereas the maximum interval over which $W$ has rank two is $[3,7]$, and these regions differ by at most $\delta=1$ in their endpoints.
Similarly, the maximum interval over which $V$ has rank at least one is over $[2,8]$, whereas the maximal intervals over which $W$ has rank at least one is over either $[1,7]$ or $[3,7.5]$: enlarging $[2,8]$ by $\delta=1$ on either endpoint covers either of these intervals in $W$, and enlarging $[1,7]$ by $\delta=1$ on either endpoint covers $[2,8]$.
This is an intuitive explanation why for all $\varepsilon$ and $\alpha$, we have
\begin{align*}
g_V(\varepsilon,\alpha+1) &= \rank(V_{\varepsilon-\alpha-1}^{\varepsilon+\alpha+1}) \le \rank(W_{\varepsilon-\alpha}^{\varepsilon+\alpha}) = g_W(\varepsilon,\alpha) \quad \text{and}\\
g_W(\varepsilon,\alpha+1) &= rank(W_{\varepsilon-\alpha-1}^{\varepsilon+\alpha+1}) \le \rank(V_{\varepsilon-\alpha}^{\varepsilon+\alpha}) = g_V(\varepsilon,\alpha).
\end{align*}

In this example, the bottleneck distance ($1.5$) is larger than the erosion distance ($1$), answering Question~\ref{ques:interleaving-fixed-t-iff} in the negative.

\end{example}

According to personal correspondence with Patel and Fasy, there is an $O(n\log n)$ algorithm for computing the erosion distance, where $n$ is the number of points in the diagram, which is much faster than computing the bottleneck distance.
While Question~\ref{ques:interleaving-fixed-t-iff} reveals that the two distances are not equivalent, the erosion distance could serve as a bound on the bottleneck distance.
The rank invariants are also M\"{o}bius inversions of persistence diagrams~\cite{patel2018generalized,mccleary2020bottleneck}.

\section{Continuity of crocker stacks}\label{sec:continuity}

If two time-varying metric spaces are close to one another, then it turns out that the resulting crocker stacks are also (in some sense) close to one another.
This is referred to as the \emph{continuity} of crocker stacks.
We explain why crocker plots are not continuous, before describing the sense in which crocker stacks are continuous.

\subsection{Discontinuity of crocker plots}
\label{ssec:discont}
We first point out that crocker plots are not continuous.

\begin{figure}[ht]
\centering
\includegraphics[width=2in]{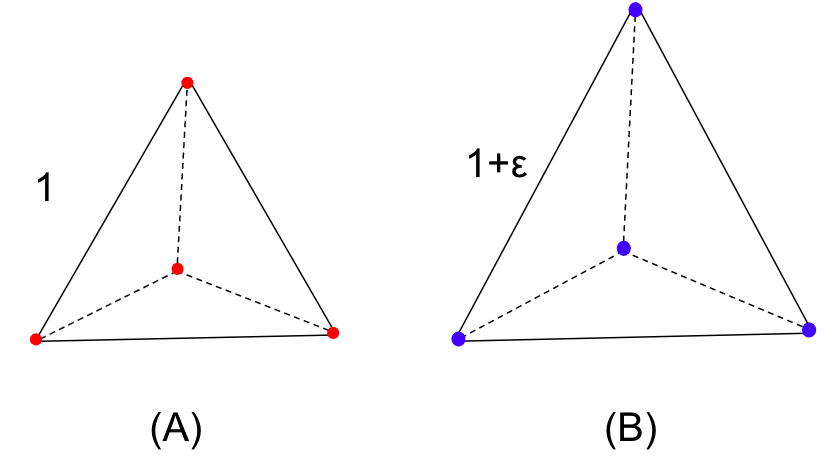}
\captionsetup{width=.9\linewidth}
\caption{Suppose $A$ and $B$ are dynamic metric spaces.
The distance between any two of the four red points in $A$ is 1 for all times $t$.
The distance between any two of the four blue points in $B$ is $1+\varepsilon$ for all times $t$.
At an arbitrary time step $t$ and scale parameter $1+\frac{\varepsilon}{2}$, we have $\beta_0^A=1$ and $\beta_0^B=4$.}
\label{fig:StabilityExample}
\end{figure}

\begin{example}\label{ex:crocker-plot-not-naively-stable}
Suppose we have two dynamic metric spaces, each with 4 points in the metric space, as shown in Figure~\ref{fig:StabilityExample}.
In the first dynamic metric space $A$, the distance between any two points is 1 for all times $t$.
In the second dynamic metric space $B$, the distance between any two points is $1+\varepsilon$ for all times $t$.
The two dynamic metric spaces are $\varepsilon$-close in the Gromov--Hausdorff distance.
However, at scale parameter $1+\frac{\varepsilon}{2}$, the first dynamic metric space $A$ at any time has 0-dimensional Betti number $\beta_0^A=1$.
Yet at any time in the second dynamic metric space, $\beta_0^B=4$ at this same scale value.
Thus, these Betti numbers differ by $\beta_0^B-\beta_0^A=4-1=3$ at scale $1+\frac{\varepsilon}{2}$.
By increasing $n=4$ points to (say) $n=1,000,000$ points or beyond, we can make the values of these Betti numbers as far apart as we want.
Under most notions of matrix distance, this would make the distance between the crocker plot matrices as far apart as we want, all while keeping the metric spaces within $\varepsilon$ in the Gromov--Hausdorff distance.
This example does not rely on time; it is really an example showing why Betti curves are not stable (in the traditional sense of $L_\infty$ distance between curves).
\end{example}

\subsection{Distances between time-varying metric spaces}

In order to describe the continuity properties of crocker stacks, we begin with some preliminaries on distances between time-varying metric spaces and time-varying persistence modules.

\begin{definition}
\label{def:time-GH}
Let $\bX$ and $\bY$ be continuous time-varying metric spaces over $t\in [0,T]$, and fix $1\le p\le \infty$.
\sloppy
The $p$-Gromov--Hausdorff distance between $\bX$ and $\bY$ is
\[d_\gh^p(\bX,\bY)=\left(\int_0^T d_\gh(X_t,Y_t)^p\ dt\right)^{1/p}.\]
When $p=\infty$ we have $d_\gh^\infty(\bX,\bY)=\sup_t d_\gh(X_t,Y_t)$.
\end{definition}

To see that this is well-defined, note that since $\bX$ and $\bY$ are continuous, the function $[0,T] \to \R$ defined by $t \mapsto d_\gh(X_t,Y_t)^p$ is a continuous function over a closed interval, and hence is integrable.

We remark that this is only a pseudo-metric, not an actual metric.
Indeed, as pointed out in Figure~1 of~\cite{kim2020spatiotemporal}, two distinct time-varying metric spaces that are not qualitatively similar can have Gromov--Hausdorff distance zero from each other at each time $t$.
The distances between multiparameter rank functions introduced in~\cite{kim2020spatiotemporal} are also stable with respect to more refined notions of distance between time-varying metric spaces.
In this paper, we restrict attention to the weaker Definition~\ref{def:time-GH} and show that crocker stacks, which are amenable for machine learning tasks, are furthermore a continuous topological invariant.

\subsection{Distances between time-varying persistence modules}

We define an $L_p$ bottleneck distance between corresponding time-varying persistence modules.

\begin{definition}\label{def:vineyard-dist}
Let $\bV$ and $\bW$ be continuous time-varying persistence modules over $t\in [0,T]$, and fix $1\le p\le \infty$.
\sloppy
The $p$-bottleneck distance between $\bV$ and $\bW$ is $d_b^p(\bV,\bW)=(\int_0^T d_b(V_t,W_t)^p\ dt)^{1/p}$.
When $p=\infty$ we have $d_b^\infty(\bV,\bW)=\sup_t d_b(V_t,W_t)$.
\end{definition}

Since $\bV$ and $\bW$ are continuous, the function $[0,T] \to \R$ defined by $t \mapsto d_b(V_t,W_t)^p$ is a continuous function over a closed interval, and hence is integrable.

Recall that our time-varying metric spaces $\bX$ are defined to have the property that $X_t$ is compact for all $t\in [0,T]$.

\begin{lemma}\label{lem:gh}
If $\bX$ and $\bY$ are continuous time-varying metric spaces over $t\in [0,T]$, then for all $1\le p\le \infty$ we have $d_b^p(\ph(\vrp{\bX}),\ph(\vrp{\bY}))\le 2d_\gh^p(\bX,\bY)$.
\end{lemma}
An analogous bound is true if Vietoris--Rips complexes are replaced with \v{C}ech complexes.

\begin{proof}
Let $p<\infty$.
For any $t\in[0,T]$, we have $d_b(\ph(\vrp{X_t}),\ph(\vrp{Y_t})) \le 2d_\gh(X_t,Y_t)$ by the stability of persistent homology.
Integrating over all $t\in[0,T]$ gives
\begin{align*}
&d_b^p(\ph(\vrp{\bX}),\ph(\vrp{\bY}))=\Biggl(\int_0^T d_b\Bigl(\ph(\vrp{X_t}),\ph(\vrp{Y_t})\Bigr)^p\ dt\Biggr)^{1/p} \\
\le&\Biggl(\int_0^T \Bigl(2d_\gh(X_t,Y_t)\Bigr)^p\ dt\Biggr)^{1/p}
=2\Biggl(\int_0^T d_\gh(X_t,Y_t)^p\ dt\Biggr)^{1/p} 
=2d_\gh^p(\bX,\bY).
\end{align*}
The same proof works for $p=\infty$ by replacing integrals with supremums.
Indeed,
\begin{align*}
&d_b^\infty(\ph(\vrp{\bX}),\ph(\vrp{\bY}))=\sup_t d_b\Bigl(\ph(\vrp{X_t}),\ph(\vrp{Y_t})\Bigr) \\
\le&2\sup_t d_\gh(X_t,Y_t)
=2d_\gh^{\infty}(\bX,\bY).
\end{align*}
\end{proof}

Lemma~\ref{lem:gh} gives a notion of continuity for stacked sequences of persistence diagrams.
By contrast, the vines in a vineyard are not stable as curves in time, as explained in Section~\ref{ssec:vineyards}.

\subsection{Continuity of crocker stacks}
\label{ssec:continuity-stack}

The continuity of a stacked set of persistence diagrams in the section above implies a continuity result for crocker stacks.
Recall from Definition~\ref{def:crocker-stack} that if $\bV$ is a time-varying persistence module, then
\[f_\bV(t,\varepsilon,\alpha) := \rank\left(V_t(\varepsilon-\alpha)\to V_t(\varepsilon+\alpha)\right)=g_{V_t}(\varepsilon,\alpha).\]

Suppose two time-varying metric spaces $\bX$ and $\bY$ have the property that $X_t$ and $Y_t$ are within Hausdorff distance $2\delta$ at all times $t$.
The interleaving distance between the persistence modules $\ph(\vrp{X_t})$ and $\ph(\vrp{Y_t})$
is at most $\delta$ at all times $t$.
Hence, for all $t$, $\varepsilon$, and $\alpha$, we have that $f_\bX(t,\varepsilon,\alpha+\delta)\le f_\bY(t,\varepsilon,\alpha)$ and $f_\bY(t,\varepsilon,\alpha+\delta)\le f_\bX(t,\varepsilon,\alpha)$, where by an abuse of notation we let $f_\bX$ and $f_\bY$ denote $f_{\ph(\vrp{\bX})}$ and $f_{\ph(\vrp{\bY})}$.
This can be thought of as a notion of continuity, since it says the crocker stack $f_\bX$ for $\bX$ is in some sense ``close" to the crocker stack $f_\bY$ for $\bY$.
In this subsection, we provide proofs for these observations.

\begin{lemma}\label{lem:interleaving}
If $\bV$ and $\bW$ are time-varying persistence modules, and if $d_b^\infty(\bV,\bW)\le \delta$, then for all $t$, $\varepsilon$, and $\alpha$ we have
\begin{itemize}
\item $f_\bV(t,\varepsilon,\alpha+\delta)\le f_\bW(t,\varepsilon,\alpha)$,
and
\item $f_\bW(t,\varepsilon,\alpha+\delta)\le f_\bV(t,\varepsilon,\alpha)$.
\end{itemize}
\end{lemma}

\begin{proof}
By hypothesis, we have $d_b(V_t,W_t)\le\delta$ for all $t\in[0,T]$.
From Lemma~\ref{lem:interleaving-fixed-t} we have
$g_{V_t}(\varepsilon,\alpha+\delta)\le g_{W_t}(\varepsilon,\alpha)$
and
$g_{W_t}(\varepsilon,\alpha+\delta)\le g_{V_t}(\varepsilon,\alpha)$ for all $t\in[0,T]$.
The conclusion follows since, by definition, we have $f_\bV(t,\varepsilon,\alpha)=g_{V_t}(\varepsilon,\alpha)$ and $f_\bW(t,\varepsilon,\alpha)=g_{W_t}(\varepsilon,\alpha)$ for all $t$, $\varepsilon$, and $\alpha$.
\end{proof}

Note that a version of the above lemma for $d_b^p$ instead of $d_b^\infty$ would not give inequalities for each $t$ but instead a single pair of inequalities that are integrated (in an $L_p$ sense) over all $t$.

The following theorem says that if the time-varying metric spaces $\bX$ and $\bY$ are nearby, then their crocker stacks are also close.
Recall that our time-varying metric spaces are defined to be compact at each point in time.

\begin{theorem}[Continuity theorem for crocker stacks]\label{thm:continuity}
If $\bX$ and $\bY$ are time-varying metric spaces, and if $d_\gh^\infty(\bX,\bY)\le\delta/2$, then the crocker stacks for $\bX$ and $\bY$ are close in the sense that for all $t$, $\varepsilon$, and $\alpha$, we have
\begin{itemize}
\item $f_\bX(t,\varepsilon,\alpha+\delta)\le f_\bY(t,\varepsilon,\alpha)$,
and
\item $f_\bY(t,\varepsilon,\alpha+\delta)\le f_\bX(t,\varepsilon,\alpha)$.
\end{itemize}
\end{theorem}

\begin{proof}
By Lemma~\ref{lem:gh} we have that $d_b^\infty(\ph(\vrp{\bX}),\ph(\vrp{\bY}))\le\delta$, and hence the conclusion follows from Lemma~\ref{lem:interleaving} with $\bV=\ph(\vrp{\bX})$ and $\bW=\ph(\vrp{\bY})$.
\end{proof}

An analogous result is true if crocker stacks are defined using \v{C}ech complexes in place of Vietoris--Rips complexes.

Though crocker stacks are continous in the above sense, we want to acknowledge that they are not continuous when they are interpreted as vectors equipped with the Euclidean norm, which was the norm we used on crocker stacks in Section~\ref{sec:vicsek}.
More work remains to be done on effectively and stably vectorizing time-varying topological summaries for use in machine learning applications.

\section{Conclusion}\label{sec:conclusion}

In this paper, we have provided an overview of topological tools for summarizing time-varying metric spaces, developed a new tool called the crocker stack, investigated the discriminative power of topological descriptors on a parameter recovery task, and discussed notions of continuity for these representations.

The crocker plot introduced in~\cite{topaz2015topological} is a topological summary of time-varying data.
It has been shown in~\cite{topaz2015topological, ulmer2019topological,bhaskar2019analyzing} to be useful in exploratory data analysis, statistical tests, and machine learning tasks.
However, the crocker plot is not stable in any sense.
We proposed the crocker stack as an alternative, which---like the crocker plot---can be discretized and treated as a vector in Euclidean space, making it useful in machine learning. 
Yet, the crocker stack also satisfies a continuity property, albeit with respect to a non-Euclidean metric.

A crocker stack is a 3D topological representation of a time-varying metric space that, like the crocker plot, displays all times in a single frame indexed by smoothing parameter $\alpha$.
This may be better for visualization purposes than, say, the vineyard representation.
Vineyards or stacked sets of persistence diagrams would have to be further vectorized for use in machine learning, perhaps by using say, persistence images~\cite{adamsImages2017}.
The stable multiparameter rank function and distance measures in~\cite{kim2020spatiotemporal} are more discriminatory for time-varying metric spaces than the (pseudo-)distances we consider, but it is not obvious how to vectorize them for use in machine learning.
 
Through computational experiments, we show that crocker plots and stacks are more effective than the alignment order parameter, a traditional method derived from physics, at parameter identification in an ubiquitous model of biological aggregations.
However, computing crocker plots and stacks are more time-intensive than computing order parameters.
At each time step, a single number---the normalized average velocity---is computed in the order parameter, while persistent homology is computed in order to produce crocker plots and stacks.
In our experiments, computing persistent homology is roughly two orders of magnitude more expensive than computing order parameters.
The discriminative capability of the crocker representations may provide benefits that outweigh this computational complexity, however.

We end with a collection of questions and possible directions for researchers to explore.
\begin{enumerate}
\item How do different choices of metrics affect the discriminatory power of crocker stacks in Section~\ref{sec:vicsek}?
For every $1\le p\le \infty$ and $1\le q\le \infty$, there exists a $(p,q)$ metric on crocker stacks: to compare two fixed persistence diagrams, we can choose any $1\le p\le \infty$ and use an $L_p$ Wasserstein distance.
To compare two time-varying persistence diagrams, we then could weight these distances over all times by choosing any $1\le q\le \infty$ averaging.
\item What is the discriminatory power of a time-varying erosion distance say, for example, in the experiments in Section~\ref{sec:vicsek}?
\item How well do time-varying persistence images~\cite{adamsImages2017} or time-varying persistence landscapes~\cite{bubenik2015statistical} work for machine learning classification tasks?
\item What are other notions of time-varying topological invariants that are both stable and also vectorizable in a way that is useful in machine learning?
\end{enumerate}

\section*{Acknowledgments}
We are grateful to Matraiyee Deka, Brittany Terese Fasy, Tom Halverson, Michael Lesnik, Dmitriy Morozov, and Amit Patel for helpful conversations.

\bibliographystyle{AIMS}
\bibliography{XiaAdaTop2020}

\end{document}